\documentclass{article}

\usepackage{times}
\usepackage[utf8]{inputenc}

\usepackage{array}
\newcolumntype{P}[1]{>{\RaggedRight\hspace{0pt}}p{#1}}
\usepackage{enumerate}
\usepackage{array}
\usepackage{amsmath,amsthm}
\usepackage{amssymb}
\usepackage{multirow}
\usepackage[export]{adjustbox}
\usepackage{xcolor}
\usepackage{graphicx}
\usepackage{epsfig}
\usepackage{subfigure}
\usepackage{slashbox}
\usepackage{amsthm}
\usepackage{multirow}
\usepackage{listings}
\usepackage{verbatim}
\usepackage{comment}
\usepackage{epstopdf}
\usepackage[algoruled]{algorithm2e}

\usepackage{epsfig}
\usepackage{booktabs,array}
\usepackage{mathtools}

\usepackage{amsthm}
\providecommand{\keywords}[1]{\textbf{\textit{Index terms---}} #1}

\newtheorem{theorem}{Theorem}
\newtheorem{corollary}{Corollary}
\newtheorem{definition}{Definition}
\newtheorem{proposition}{Proposition}

\SetKwComment{Comment}{$\triangleright$\ }{}

\newcommand{\specialcell}[2][c]{%
	\begin{tabular}[#1]{@{}c@{}}#2\end{tabular}}

\newcommand{\specialcelll}[2][l]{%
	\begin{tabular}[#1]{@{}l@{}}#2\end{tabular}}

\newcommand{\specialcel}[2][c]{%
	\begin{tabular}[#1]{@{}l@{}}#2\end{tabular}}

\newcommand{\CommentX}[1]{\unskip~/*~#1~*/}

\usepackage[noend]{algpseudocode}
\usepackage[utf8]{inputenc}
\usepackage{soul}
\usepackage{hyperref}

\usepackage{blindtext,graphicx}
\usepackage[absolute,showboxes]{textpos}
\begin{document}
	\textblockrulecolour{blue}
\begin{textblock}{12}(5,2.5)
	 \vspace{0.3cm}
	\large
	\centering
	\noindent
	
\textcolor{blue}{For the final published version in the IEEE Transactions on 
	\newline Neural Networks and Learning Systems, please refer to \newline \url{https://ieeexplore.ieee.org/document/9185001}}
 \vspace{0.3cm}
\end{textblock}

		\title{%On
		Demystifying Deep Learning  \\ in Predictive Spatio-Temporal Analytics: \\ An Information-Theoretic Framework
	}
	\date{}
	\author{Qi Tan,
		Yang Liu,
		and~Jiming~Liu
		\thanks{This work was supported in part by the General Research Fund from the Research Grant Council of Hong Kong SAR under Projects RGC/HKBU12201318, RGC/HKBU12201619, RGC/HKBU12202220, and RGC/HKBU12202417, and in part by the Faculty Research Grant of Hong Kong Baptist University (HKBU) under Project FRG2/17-18/027.
		Jiming Liu is the corresponding author.  
		Qi Tan, Yang Liu and Jiming Liu are with the Department of Computer Science, Hong Kong Baptist University, Hong Kong SAR, China (e-mail: \{csqtan,csygliu,jiming\}@comp.hkbu.edu.hk).}% <-this % stops a space
	}
	
	\maketitle
	
	\begin{abstract}
		Deep learning has achieved incredible success over the past years, especially in various challenging predictive spatio-temporal analytics (PSTA) tasks, such as disease prediction, climate forecast, and traffic prediction, where intrinsic dependency relationships among data exist and generally manifest at multiple spatio-temporal scales. 	
		However, given a specific PSTA task and the corresponding dataset, how to appropriately determine the desired configuration of a deep learning model, theoretically analyze the model's learning behavior, and quantitatively characterize the model's learning capacity remains a mystery.
		%Predictive spatio-temporal analytics (PSTA) is of great importance in many real-world applications, such as disease prediction, climate forecast, and traffic prediction. 
		%One of the most challenging tasks in PSTA is to learn the intrinsic dependency relationships among data, which generally manifest at multiple spatio-temporal scales.
		%Several models have been proposed for such PSTA tasks, including time series models, tensor-based learning models and deep neural network models.
		%However, these models have not explicitly addressed the issue of multi-scale dependency modeling and hence have presented certain limitations in their performance.
		%Moreover, how to theoretically characterize a learning model's capacity in capturing the multi-scale dependency of spatio-temporal data remains unclear.
		In order to demystify the power of deep learning for PSTA in a theoretically sound and explainable way, in this paper, we provide a comprehensive framework for deep learning model design and information-theoretic analysis.
		First, we develop and demonstrate a novel interactively- and integratively-connected deep recurrent neural network (I$^2$DRNN) model. 
		I$^2$DRNN consists of three modules: 
		an Input module that integrates data from heterogeneous sources;
		a Hidden module that captures the information at different scales while allowing the information to flow interactively between layers; and 
		an Output module that models the integrative effects of information from various hidden layers to generate the output predictions.
		Second, to theoretically prove that our designed model can learn multi-scale spatio-temporal dependency in PSTA tasks, we provide an information-theoretic analysis to examine the information-based learning capacity (i-CAP) of the proposed model.
		In so doing, we can tackle an important open question in deep learning, that is, how to determine the necessary and sufficient configurations of a designed deep learning model with respect to the given learning datasets.
		Third, to validate the I$^2$DRNN model and confirm its i-CAP, we systematically conduct a series of experiments involving both synthetic datasets and real-world PSTA tasks. 
		The experimental results show that the I$^2$DRNN model outperforms both classical and state-of-the-art models on all datasets and PSTA tasks.
		More importantly, as readily validated, the proposed model captures the multi-scale spatio-temporal dependency, which is meaningful in the real-world context. Furthermore, the model configuration that corresponds to the best performance on a given dataset always falls into the range between the necessary and sufficient configurations, as derived from the information-theoretic analysis.
	\end{abstract}
	
	
	\section{Introduction}
	Deep learning has received remarkable attention over the past years and achieved incredible success in various applications. 
	Recently, studies of deep learning for predictive spatio-temporal analytics (PSTA)
	%, such as disease prediction~\cite{Wu18SIGIR}, climate forecast~\cite{xu2019spatio}, and traffic prediction~\cite{zhang2018predicting}, 
	have become increasingly important as their applications are closely related to human well-being, and the instrumented and interconnected world makes spatio-temporal data more ubiquitous than ever before.
	%, such as disease data, climate data, and traffic data, more ubiquitous than ever before.		
	%PSTA aims to analyze and capture the dependency among data along both spatial and temporal dimensions for future forecast, has long attracted extensive research interest~\cite{eshel2011spatiotemporal}. 
	
	One of the most challenging tasks in PSTA~\cite{eshel2011spatiotemporal} is to learn the intrinsic dependency relationships among data, as in many real-world applications 
	such as disease prediction~\cite{matsubara2014funnel,Yang:2014AAAI}, climate forecast~\cite{grover2015deep,xu2019spatio,8633392}, and traffic prediction~\cite{zhang2017deep,yao2019learning,Wang:2019:LTC:3308558.3313704,qi2020aaai}, 
	such dependency relationships are generally exhibited at multiple spatio-temporal scales among heterogeneous data sources ~\cite{gao2007multiscale,allen2017multiscale,xu2018muscat}.
	Taking infectious disease as an example, the infected case number in a specific region might be on a downward trend each year, but the actual case number for various smaller regions at different time of year may fluctuate with a multitude of factors, such as environmental, geographic, meteorological, and demographic factors, at varying scales~\cite{Yang:2014AAAI}. 
	Those governing factors, which affect the real-world in different ways, complicate the spatio-temporal dependency and make it difficult to capture.
	
	As a powerful nonlinear learner for data representation and information extraction, deep learning has demonstrated its unique ability in capturing complex spatio-temporal dependency among data for making accurate predictions.
	However, there remain certain open questions: \textbf{Given a specific PSTA task and the corresponding dataset, how to appropriately determine the desired configuration of the deep learning model, so that the useful information contained in the data can be effectively extracted; and how to theoretically analyze the model's learning behavior and quantitatively characterize the model's learning capacity, so that the outstanding performance of the model can be well explained}.
	These are the questions that we aim to answer in this paper.

	\subsection{Related Work}

	\vspace{0.0cm}
	\subsubsection{Traditional Learning Models for PSTA}
	
	\
	
	%Several models have been proposed to tackle the PSTA task, including time series models, tensor-based learning models, and deep neural network models.
	
	%Time-Series Models
	
	The earliest spatio-temporal prediction models are based on the classical time series models (e.g., vector autoregressive model~\cite{banerjee2014hierarchical} and Gaussian process regression model~\cite{senanayake2016predicting}). 
	The spatial observations in one time step are treated as a vector, and the spatial dependency can be regarded as the multivariate dependency in the time series models.
	
	%tensor models
	Some tensor-based models have been proposed to account for the underlying dependency among spatio-temporal variates from different data sources. 
	Bahadori et al.~\cite{bahadori2014fast} treated the spatio-temporal data as tensors and proposed a low-rank tensor learning framework for spatio-temporal prediction.
	Furthermore, the spatial autocorrelation~\cite{takeuchi2017autoregressive}, temporal autocorrelation~\cite{yu2016temporal}, and high-order temporal correlation~\cite{jing2018high} have been modelled as constraints and integrated into the tensor factorization frameworks for the PSTA task.

	\vspace{0.2cm}
	\subsubsection{Deep Learning for PSTA}
	
	\
	
	Various deep learning models have recently been proposed to capture the spatio-temporal dependency and predict spatio-temporal series. On the one hand, some studies treated the spatio-temporal prediction as a regression problem.
	Zhang et al.~\cite{zhang2017deep,zhang2018predicting} proposed a spatio-temporal residual network (ST-ResNet) that integrates the temporal closeness, period and trend properties of the target data and some external features for regression. 
	Yao et al.~\cite{Yao2018AAAI} considered spatial and temporal information as different views and proposed a deep multi-view spatio-temporal network framework to predict taxi demand.
	
	Recurrent neural network (RNN)-based models have also been widely adopted to model spatio-temporal datas.
	Ziat et al.~\cite{ziat2017spatio} proposed a neural hidden state model in which the transition between states is modeled using a neural network. The attention mechanism has recently been considered in capturing the context-dependent dependency~\cite{qin2017dual,liang2018geoman}. RNNs can also be regarded as a hidden state model with rich representation capacity~\cite{karpathy2015visualizing,collins2016capacity,Quan2019Recurrent,Kerg2019non}.

	\vspace{0.2cm}
	\subsubsection{Theoretical Understanding on Deep Learning}
	
	\
	
	In order to explore the learning behavior of deep learning, some studies attempted to understand the deep models from the theoretical perspective.
	Tishby and Zaslavsky investigated the deep neural networks based on the information bottleneck (IB) principle~\cite{tishby2015deep}. 
	%Following the IB principle, Schwartz-Ziv and Tishby further demonstrated the effectiveness of the Information Plane visualization of deep neural networks~\cite{shwartz2017opening}.	
	%Saxe et al., also following the IB principle, disagreed the conclusions drawn in~\cite{shwartz2017opening}: they argued that there was no independent compression phase in the training of deep neural networks with linear activation functions and single-sided saturating nonlinearities, and the compression didn't relate to the generalization performance nor the diffusion-like behavior of stochastic gradient descent~\cite{saxe2018information}.	
	%Goldfeld et al. utilized the IB principle to describe the mutual information between the input of a deep neural network and the output vector of its hidden layer~\cite{goldfeld19a}.
	Following the IB principle, Goldfeld et al. described the mutual information between the input of a deep neural network and the output vector of its hidden layer~\cite{goldfeld19a}.
	In addition to investigating the deep neural networks via information, researchers also studied deep learning from other perspectives, e.g., physics and geometry.
	%Lei et al. connected deep neural networks with the physical system, examined deep learning from the microscopic, macroscopic, and physical world views, and described the learning behaviors of deep neural networks using physics principles~\cite{lei2018opening}.

	\subsection{Motivation}
	Although existing models aim to tackle the PSTA task, they have not explicitly addressed the issue of multi-scale dependency modeling and hence have presented certain limitations in their performance.
	The classical time series models, tensor-based models, and regression-based deep neural networks are mainly derived from the autoregressive model, which emphasizes the short-range temporal dependency but ignores the long-term dependency to some extent. 
	RNNs are promising in processing sequential data, and in particular, stacked RNNs can learn a hierarchical representation of information. However, stacked RNNs do not allow the feedback information to flow from the top layers to the bottom layers, nor is there a direct connection between the output layer and the non-top layers, which weakens the memorization of information at different scales and the effect of information from the bottom/intermediate layers in generating predictions.
	
	More importantly, even though the existing deep learning models have achieved good performance in some given PSTA tasks, and earlier studies have been conducted to examine the learning behaviors of deep neural networks, it still lacks an in-depth understanding on why a specific deep learning model works well on a given dataset and what is the relationship between the configuration of a deep learning model and its corresponding learning capacity on the given task and dataset.
	Without such a clear understanding, it remains a mystery how to determine the desired configurations of a certain deep model with respect to the given learning datasets, and thus making it difficult to sustain the success of deep learning.
	
	To address the above unsolved yet challenging issues in a theoretically sound and explainable way, it is of great importance to answer three questions: 
	\begin{enumerate}
		\item How can we design a learning model to characterize the complex multi-scale dependency of spatio-temporal data?
		\item How can we quantify and analyze the designed model's capacity in capturing the multi-scale dependency of spatio-temporal data?
		\item How can we validate the learning behavior and performance of the designed model in various scenarios of multi-scale spatio-temporal dependency?
	\end{enumerate}

	\subsection{Our Contributions}
	This paper is aimed to specifically tackle the challenge of demystifying deep learning in PSTA by answering the three questions above. 
	The contributions of the paper can be highlighted as follows: 
	\begin{enumerate}
	
		\item 	We propose  a novel interactively- and integratively-connected deep recurrent neural network (I$^2$DRNN) model to answer the first question. 
		I$^2$DRNN  is composed of three important modules: an Input module to integrate heterogeneous data, a Hidden module to allow information interaction between layers, and an Output module to integrate the information from varying scales to make predictions. With the integration of these modules, I$^2$DRNN can model the integrative effects of varying scales of spatio-temporal data, within and/or across diverse factors, as observed from heterogeneous sources. The designed I$^2$DRNN model provides the basis for our subsequent information-theoretic analysis.
		
		\item We develop an information-theoretic framework to answer the second question. This framework enables us to theoretically analyze I$^2$DRNN's learning behavior, to quantitatively characterize the information-based learning capacity (i-CAP) of each component of I$^2$DRNN in terms of capturing the multi-scale spatio-temporal information, and to appropriately determine the necessary and sufficient configurations for I$^2$DRNN with respect to a given dataset. With the information-theoretic guarantees, the developed framework serves as a rigorous and explainable guidance for designing a desirable deep architecture for a given learning task.
	
		\item We systematically design a series of experiments on both synthetic datasets with multi-scale dependency and real-world PSTA tasks with heterogeneous data sources to answer the third question. I$^2$DRNN achieves better performance than existing models on all datasets and PSTA tasks. The multi-scale spatio-temporal dependency captured by I$^2$DRNN can be interpreted in a real-world context. Moreover, on all datasets and PSTA tasks, the model configuration that achieves the best performance always falls within the interval between the necessary configuration and the sufficient configuration, which is consistent with our information-theoretic analysis.
	\end{enumerate}

	\subsection{Organization of the Paper}
	The rest of the paper is organized as follows. 
	Section \ref{SEC:HIER} elaborates the details of our model design, including the problem definition, preliminaries in RNN, and the proposed
	I$^2$DRNN model.
	Section \ref{SEC:ITRNN} presents the information-theoretic framework to analyze the i-CAP of our model and to determine the appropriate model configurations with respect to a given dataset.
	Section \ref{SEC:SYE} provides extensive experimental results on both synthetic datasets and real-world PSTA tasks to validate the effectiveness of the proposed I$^2$DRNN model and to confirm its i-CAP derived from the information-theoretic analysis.
	Section \ref{SEC:CON} concludes this paper.

	\section{Model Design}
	\label{SEC:HIER}
	In this section, we provide the details of the deep model design, serving as the basis of our information-theoretic analysis.
	First, we formally define the problem of multi-scale spatio-temporal dependency learning in PSTA, including the variables and formulations. 
	We then present some necessary preliminaries in RNN, which are the foundation of the proposed model.
	After that, we propose our I$^2$DRNN model, explain its architecture as well as the role of each designed module in capturing the multi-scale dependency among spatio-temporal data, and present the learning procedure.
	
	\subsection{Problem Definition}
	\label{SEC:PF}
	Let $\mathbf{Y}\in \Re^{N \times T}$ be the target variable collected in $N$ locations during $T$ time steps and $\mathbf{y}_t \in \Re^{N}$ be the target variable at time step $t$ ($t = 1, \cdots, T$). 
	Let $\mathbf{X}^{(i)} \in \Re^{N^{(i)} \times T^{(i)}}$ ($i=1,2,\cdots,d$) be the covariate observed from the $i$-th data source, where $d$ denotes the number of data sources. Note that the spatial and temporal resolutions of these data are not necessarily well aligned. 
	The target of multi-scale spatio-temporal dependency learning in PSTA is to learn such a function $\mathbf{f}(\cdot)$ that reflects the complex intrinsic relationships among the target variable and multiple covariates, so that $\mathbf{y}_t$ can be well predicted with the learned function $\mathbf{f}(\cdot)$ and the input of historical observations on $\mathbf{Y}$ and observed covariates: 
	\begin{equation}
	\small
	\mathbf{y}_t = \mathbf{f}(\mathbf{Y}_{[1:t-1]}, q^{(i)}(\mathbf{X}^{(i)}_{[1:t]})|_{i=1}^{d}),%+\epsilon_t,
	\label{EQU:PF}
	\end{equation}
	where $\mathbf{Y}_{[1:t-1]}$ denotes the target variable collected from time steps $1$ to $t-1$, 
	$\mathbf{X}^{(i)}_{[1:t]}$ denotes the covariate observed from the $i$-th data source up to time step $t$, and $q^{(i)}(\mathbf{X}^{(i)}_{[1:t]})$ indicates the effect of $i$-th covariate on the target variable. 
	The notations used in this paper are described and explained in Table~\ref{TAB:Notation}.
	
	\begin{table}[!t]
		\renewcommand{\arraystretch}{1.3}
		\vspace{0.00cm}
		\caption{Notations and descriptions.}	
		\label{TAB:Notation}
		\centering	
		\begin{tabular}{|l|p{6.5cm}|}
			\hline
			
			\hline
			
			\hline
			Notations & Descriptions \\
			\hline 
			
			\hline 
			
			\hline 
			$\mathbf{Y}$ & Target variable \\ \hline
			$\mathbf{X}$ & Input covariates\\ \hline
			$\mathbf{Y}_{[1:t-1]}$ & Target variable from time step $1$ to time step $t-1$\\ \hline
			$\mathbf{X}_{[1:t-1]}$ & Input covariates from time step $1$ to time step $t-1$\\ \hline
			$\mathbf{y}_t$ & Target variable in time $t$\\ \hline
			$\mathbf{x}_t$ & Input covariates in time $t$\\ 
			\hline 
			
			\hline 
			
			\hline 
			$\mathbf{h}_t^l $ & Hidden state of the $l$-th level at time step $t$, $l\in [1, L]$ \\ \hline
			$\mathbf{o}_t$ & Output value of RNN models at time step $t$\\ \hline
			$\mathbf{W}^l$ & State-to-state transition matrix of $l$-th layer in stacked RNN\\ \hline
			$\mathbf{U}^l$ & Input-to-state weight matrix of $l$-th layer in stacked RNN \\ \hline
			$\mathbf{V}$ & State-to-output weight matrix \\		 \hline
			$\mathbf{W}^{i \rightarrow j}$ & State-to-state transition matrix from $i$-th layer in time step $t-1$ to $j$-th layer at time step $t$ in I$^2$DRNN\\ \hline
			$\mathbf{V}^{l \rightarrow O} $ &State-to-output weight matrix of $l$-th layer in I$^2$DRNN \\ \hline
			$\mathbf{X}^F_t $ & Fine-scale features at time step $t$ \\ \hline
			$\mathbf{x}^C_t$ &Coarse-scale features at time step $t$\\ \hline
			$\mathbf{x}^S_t$ &The same scale features at time step $t$\\ \hline
			$\mathbf{c}_t$ & Context features\\
			\hline 
			
			\hline 
			
			\hline 
			$I(\mathbf{y};\mathbf{x})$ & Mutual information (MI) between two variables $\mathbf{y}$ and $\mathbf{x}$\\ \hline
			$D^{R}(\tau)$ & Recurrent Information Rate with time tag $\tau$ \\ \hline
			$\lambda$ & Recurrent coefficient\\ \hline
			$D^{X}$ & Input Information Rate \\ \hline
			$\eta$ & Input coefficient\\ \hline
			$I(\mathbf{y}_t; \mathbf{o}_t)$ & i-CAP of a model for a given dataset\\
			\hline 
			
			\hline 
			
			\hline 		
		\end{tabular}		
	\end{table}

	\begin{figure*}[htp]
		\centering
		\subfigure{\centerline{	\includegraphics[width=1.3\linewidth]{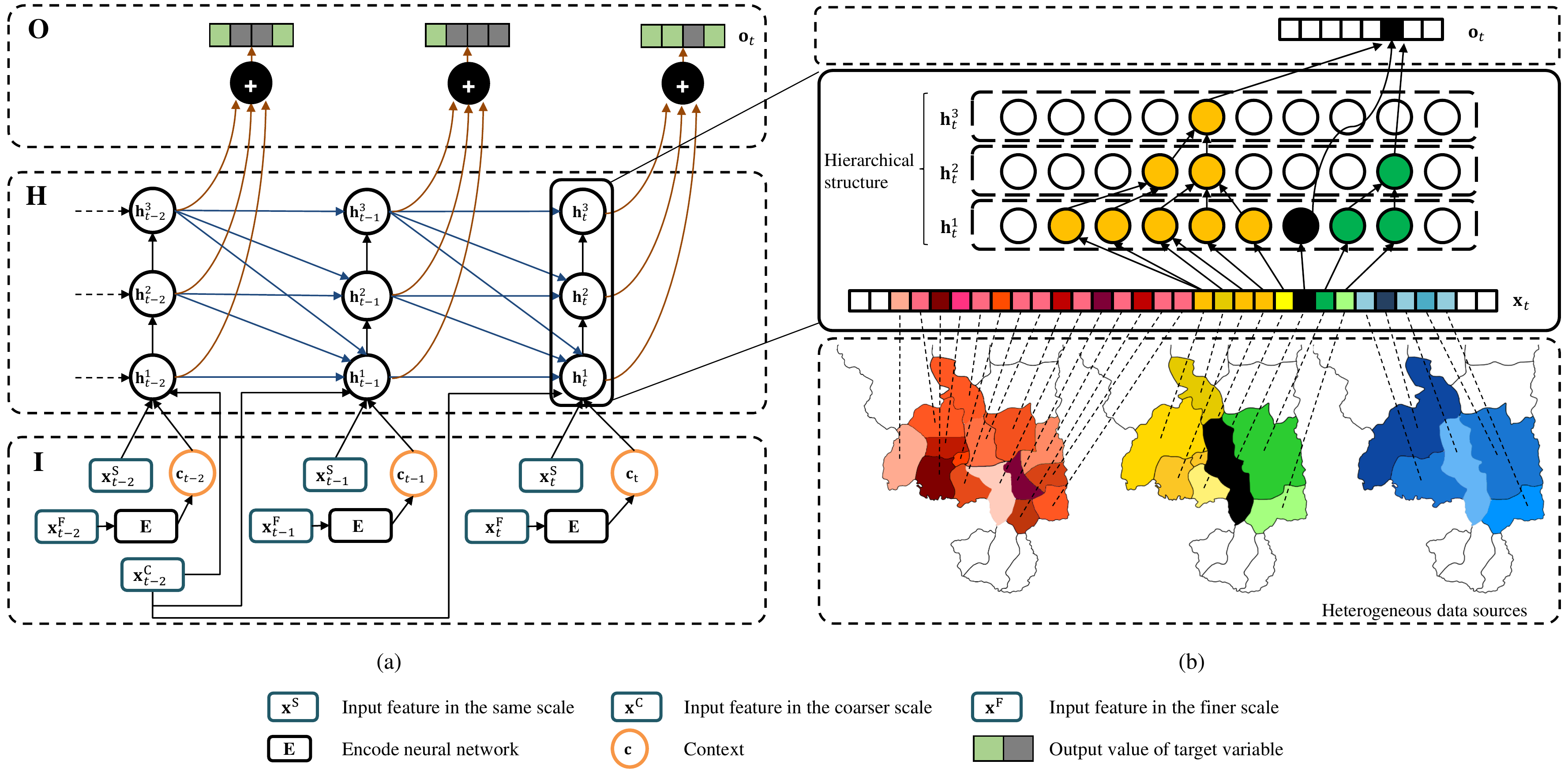}}}
		\caption{The proposed {\bf I}nteractively- and {\bf I}ntegratively-connected {\bf D}eep {\bf R}ecurrent {\bf N}eural {\bf N}etwork (I$^2$DRNN) model. 
			(a) I$^2$DRNN is composed of the Input (I) module, the Hidden (H) module, and the Output (O) module. 
			The encoder and decoder structures in I module integrate data from heterogeneous sources.
			The hierarchical structure in H module is used to capture multiple spatio-temporal effects on target variable caused by covariates from different data sources by allowing interaction of information among various layers.
			The integrative effects at varying scales are then modeled in O module to generate the output predictions.				
			(b) $\mathbf{x}_t$ is a vector that represents the data from multiple heterogeneous sources (denoted by different maps) in all different locations at time step $t$ and $\mathbf{o}_t$ is an $N$-dimensional vector representing the predicted values of target variable in $N$ locations at time $t$. By extracting the information from $\mathbf{x}_t$ using the hidden layers, i.e, $\mathbf{h}_t^1$, $\mathbf{h}_t^2$, and $\mathbf{h}_t^3$, the spatial dependency of various locations can be captured. 
			Specifically, the target variable in one location can be influenced by the effects from individual locations (e.g., the black node in $\mathbf{h}_t^1$) or the collective effects from different locations at multiple scales (e.g., the green node in $\mathbf{h}_t^2$ and the yellow node in $\mathbf{h}_t^3$). The hierarchical structure can learn such multi-scale spatial dependency. Note that the hierarchical layers are fully connected (i.e., $\mathbf{x}_t \rightarrow \mathbf{h}_t^1$, $\mathbf{h}_t^1 \rightarrow \mathbf{h}_t^2$, $\mathbf{h}_t^2 \rightarrow \mathbf{h}_t^3$, $\mathbf{h}_t^1 \rightarrow \mathbf{o}_t$, $\mathbf{h}_t^2 \rightarrow \mathbf{o}_t$ and $\mathbf{h}_t^3 \rightarrow \mathbf{o}_t$). Some of the connections are highlighted simply for illustration.}
		\vspace{-0.2cm}
		\label{Figure:Framework}
	\end{figure*}

	\subsection{Preliminaries}
	\label{SEC:PRE}
	%Because our work is based on the RNN, we briefly introduce the notations and formulations in RNN models for predictive analytics. 
	
	RNN is a typical neural network model that has been widely used in sequential prediction~\cite{Mohajerin2019Multistep}.
	By forming a directed cycle between hidden units, the historical information of the input sequences is well preserved in RNN.
	The state of the hidden unit of a conventional RNN at time step $t$ is computed as a
	function of the current input $\mathbf{x}_t$ and the previous
	hidden state $\mathbf{h}_{t-1}$:
	\begin{equation}
	\small
	\mathbf{h}_t = f_h(\mathbf{h}_{t-1}, \mathbf{x}_t).
	\end{equation}
	It is common to adopt the element-wise nonlinear activation function as the transition function: 
	\begin{equation}
	\small
	\begin{split}
	\mathbf{h}_t & = \sigma_h(\mathbf{W} \mathbf{h}_{t-1} + \mathbf{U} \mathbf{x}_t),\\
	\mathbf{o}_t & = \sigma_o(\mathbf{V} \mathbf{h}_t ),
	\end{split}
	\end{equation}
	where $\mathbf{W}$ is the state-to-state transition matrix, $\mathbf{U}$ is the input-to-state weight matrix, $\mathbf{V}$ is the state-to-output weight matrix, $\sigma_h(\cdot), \sigma_o(\cdot)$ are the element-wise activation functions, and $\mathbf{o}_t$ is the output value of the RNN model.
	
	In practice, it is difficult for a single-layer RNN to represent the complex distribution, whilst the deep structure is desirable to capture more information via multiple hidden states. The stacked RNN, an RNN model with a deep structure, organizes the multiple hidden states in a hierarchical manner:
	\begin{equation}
	\small
	\mathbf{h}_t^l = f_h^l( \mathbf{h}_t^{l-1}, \mathbf{h}_{t-1}^l) = \sigma_h(\mathbf{W}^l \mathbf{h}_t^{l-1} + \mathbf{U}^l \mathbf{h}_{t-1}^l),
	\label{EQU:stackedstructure}
	\end{equation}
	where $\mathbf{h}_t^l $ is the hidden state of the $l$-th level at time $t$, $l\in [1, L]$. 
	When $l = 1$, the state is computed using the input $\mathbf{x}_t$.
	The hidden states of all levels are recursively computed from the bottom level $l = 1$.
	This architecture can perform hierarchical processing of the temporal data and capture the structure of time series.
	Empirical evaluations have demonstrated the effectiveness of the deep structure in RNN~\cite{graves2013speech}.

	\subsection{The I$^2$DRNN Model}
	%In this subsection, we present the details of the proposed I$^2$DRNN model for multi-scale spatio-temporal dependency learning.
	Fig.~\ref{Figure:Framework}(a) illustrates the architecture of the proposed I$^2$DRNN, which is composed of three key modules: the Input (I) module that integrates heterogeneous data sources via the encoder and decoder structures, the Hidden (H) module that captures the information on various scales and allows information interaction among layers via the hierarchical structure, and the Output (O) module that integrates effects at varying scales to generate the output predictions.	
	
	\vspace{0.2cm}
	\subsubsection{Input Module: Integration of Heterogeneous Data}
	
	\
	
	To handle the heterogeneity of multi-source data, we adaptively incorporate data from various sources into the I module of I$^2$DRNN. 
	Heterogeneous data sources describe multiple dynamic processes in various scales. 
	We consider three kinds of scales: the coarse-scale, the same-scale, and the fine-scale. Here the coarse-, same-, and fine-scales are determined with respect to the temporal scale of the target variable.
	The coarse-scale process serves as the context/condition of the fine-scale process, while the dynamics of the fine-scale process reflect the state of the coarse-scale process. 
	In other words, when predicting the fine-scale process, coarse-scale data are fed in as the context variable, and when predicting the coarse-scale process, fine-scale data are used to construct the state of the coarse-scale process.
	
	The bottom of Fig.~\ref{Figure:Framework}(a) shows the encoder and decoder structures designed in the I module of I$^2$DRNN to process the data of various scales. 
	The encoder transforms a sequence into a vector representation, while the decoder generates a sequence output, given a vector representation as the input~\cite{sutskever2014sequence}. 
	In our model, we use the encoder structure for fine-to-coarse inference:
	\begin{equation}
	\small
	\mathbf{c}_t = f_E(\mathbf{X}^F_t),
	\end{equation} 
	where $\mathbf{X}^F_t \in \Re^{N^f \times T^f}$ are the fine-scale features and $f_E$ is an encoder RNN. Note that the fine-scale processes progress multiple time steps during the coarse time interval $t$.
	%Meanwhile, we use the decoder structure in the coarse-to-fine inference. 
	The coarse features are fed in repeatedly to predict the target variable. 
	We use $\mathbf{x}^C_t$ to denote the coarse-scale features.
	
	Moreover, the heterogeneous data may not be grid-distributed nor well aligned in a unified spatial resolution. The bottom of Fig.~\ref{Figure:Framework}(b) gives an example of spatial data from various resources (denoted by different maps) that are not well aligned. 
	We concatenate all heterogeneous spatial data that in the same temporal scale to form the input vectors in the coarse-scale (denoted as $\mathbf{x}^C_t$), same-scale (denoted as $\mathbf{x}^S_t$), or fine-scale (denoted as $\mathbf{x}^F_t$), respectively.
	
	Combining the feature representations from heterogeneous data sources, the input layer is constructed as follows:
	\begin{equation}
	\small
	\mathbf{x}_t = concatenate(\mathbf{c}_t, \mathbf{x}^C_t, \mathbf{x}_t^S),
	\label{EQU:xt}
	\end{equation}
	where $concatenate(\mathbf{c}_t, \mathbf{x}^C_t, \mathbf{x}_t^S)$ denotes the operation of concatenating $\mathbf{c}_t$, $\mathbf{x}^C_t$, and $\mathbf{x}_t^S$ into a single vector $\mathbf{x}_t$.

	\vspace{0.2cm}
	\subsubsection{Hidden Module: Hierarchical Structure for Multi-scale Information Interaction}
	
	\
	
	To learn the multi-scale spatio-temporal dependency, the H module should be able to model the interaction and integration of information from various layers.
	
	With information flowing over multiple time scales, temporal data often display hierarchical properties. The fast-moving component dominates the fine-scale dynamics, whilst the slow-moving component plays an important role in the coarse-scale dynamics.
	The observation is the integrative effects of information over multiple scales. 
	As demonstrated in some representative works~\cite{el1996hierarchical,koutnik2014clockwork}, the multi-scale representation of context can be learned by the hierarchical RNN with each layer working on a certain time scale. 
	
	The middle of Fig.~\ref{Figure:Framework}(a) shows the design of the H module in I$^2$DRNN that provides the feedback connections in each layer to adaptively capture the multi-scale information in the spatio-temporal dynamic processes.
	The hidden state of the $j$-th layer in Eq.~(\ref{EQU:stackedstructure}) is computed as follows:
	\begin{equation}
	\small
	\mathbf{h}^j_t = \sigma_h(\mathbf{U}^{j-1 \rightarrow j} \mathbf{h}_t^{j-1}+ \sum_{i=j}^{L} \mathbf{W}^{i \rightarrow j}\mathbf{h}^i_{t-1}), \label{EQU:hjt}
	\end{equation}
	where the superscript ${i \rightarrow j}$ indicates the gate from layer $i$ in time step $t-1$ to layer $j$ in time step $t$, and we have $\mathbf{h}^0_t = \mathbf{x}_t$.
	Organizing the RNN at multiple time scales helps to capture the long-term dependency by allowing information to flow a long distance more easily at the coarser time scales~\cite{goodfellow2016deep}. 
	Moreover, the feedback connections can decrease the variance of the current input to preserve the historical information.
	
	For the multi-scale spatial dependency, we use a fully connected hierarchical structure in I$^2$DRNN to characterize it.
	%Similarly, the structural spatial effects form the multi-scale patterns in the spatial domain. 
%The variable in one location can be influenced by those in other locations on different scales.
%Moreover, many spatial data are not grid-distributed but only have geographical information attached, such as the air quality recorded from stations at specific locations. 
	%Therefore, we use the fully connected hierarchical structure in I$^2$DRNN to characterize the multi-scale spatial dependency. 
	
	\vspace{0.2cm}
	\subsubsection{Output Module: Integrative Effects at Multiple Scales}
	
	\
	
	The spatio-temporal observation is the combined effect of the dynamics at various scales. 
	For instance, the target variable in one location (shown in black in layer $\mathbf{o}_t$ in Fig.~\ref{Figure:Framework}(b)) can be influenced by the effects from individual locations (e.g., the black node in the layer $\mathbf{h}_t^1$ in Fig.~\ref{Figure:Framework}(b)) and by the collective effects from various locations on multiple scales (e.g., the green node in layer $\mathbf{h}_t^2$ and the yellow node in layer $\mathbf{h}_t^3$ in Fig.~\ref{Figure:Framework}(b)). 
	To model such a phenomenon, the output of our model in time step $t$ should be from the integrative effects of information at all scales. 
	The top of Fig.~\ref{Figure:Framework}(a) shows the design of the O module in I$^2$DRNN that provides a connection from each hidden layer to the output layer. 
	As a result, the output at time step $t$ is computed as follows:
	\begin{equation}
	\small
	\label{EQU:OUTPUT}
	\mathbf{{o}}_t= \sum_{l= 1}^{L} \mathbf{V}^{l \rightarrow O} \mathbf{h}_t^l.
	\end{equation}
	
	\vspace{0.2cm}
	\subsubsection{The Learning Procedure}
	\label{SEC:LEARN}
	
	\
	
	Finally, we introduce the learning procedure of the proposed model.
	As described in Algorithm~\ref{ALG:ITERATIVEINF}, we
	first construct the training dataset and the validation dataset from the original data sources. The proposed model is trained to predict the target variable $\mathbf{y}_t$. Therefore, the objective is to determine the parameters of the proposed model that can minimize the difference between the predicted values and the ground truth:
	\begin{equation}
	\small
	\mathop{\arg\min}\limits_{\substack{\Theta}} \text{ } Loss(\mathbf{Y}, \Theta) = \sum_{\mathbf{y}_t \in \mathbf{Y}}||\mathbf{y}_t - \mathbf{o}_t ||^2_2,
	\end{equation}
	where $\Theta$ is the parameter set of the model, and $\mathbf{o}_t$ is the output of I$^2$DRNN. The model is trained via Adam \cite{ADAM}. 
	%We stop training until the loss on the validation set no longer decreases.
	
	\begin{algorithm}[!t]\small

		\SetKwInOut{Input}{Input}
		\SetKwInOut{Output}{Output}
		
		%\Input
		\KwIn{Historical observations: $\mathbf{Y}$;
			heterogeneous data sources: $\mathbf{X}^S, \mathbf{X}^F, \mathbf{X}^C$;}
		\KwOut{Learned parameters of I$^2$DRNN: $\Theta$}	
		%\Output
		\vspace{0.2cm}
		$\mathbf{X} \leftarrow  [ \mathbf{X}^S, \mathbf{X}^F, \mathbf{X}^C] $; 
		
		$\mathbf{\mathcal{D}} \leftarrow  <\mathbf{X}, \mathbf{Y}>$; 
		
		$\{\mathbf{\mathcal{D}}^t, \mathbf{\mathcal{D}}^v\} \leftarrow \mathbf{\mathcal{D}} $;

		$\mathbf{v}^l \leftarrow \emptyset$; 
		
		\While{stopping criteria is not satisfied}  
		{
			\For{ $t = 1 : T$ }{
				\CommentX{Compute Model Prediction}
				
				$\mathbf{c}_t \leftarrow f_E(\mathbf{X}^F_t)$
				
				$\mathbf{x}_t \leftarrow \text{concatenate}(\mathbf{c}_t, \mathbf{x}^C_t, \mathbf{x}^S_t, \mathbf{x}^{sp}_t)$;
				
				\For{$j = 1 : L$}{
					$\mathbf{h}^j_t \leftarrow \sigma_h(\mathbf{U}^{j-1 \rightarrow j} \mathbf{h}_t^{j-1} + \displaystyle{\sum_{i=j}^{L}} \mathbf{W}^{i \rightarrow j} \mathbf{h}^i_{t-1})$;  	
				}
				
				$\mathbf{{o}}_t \leftarrow \sum_{l= 1}^{L} \mathbf{V}^{l \rightarrow O} \mathbf{h}_t^l$; 
			}
			
			\CommentX{Update Parameters}
			
			$Loss(\mathcal{D}^t, \Theta) \leftarrow \sum_{\mathbf{y}_t \in \mathcal{D}^t}||\mathbf{y}_t - \mathbf{o}_t ||^2_2$;
			
			$\Theta \leftarrow \Theta - \eta \frac{\partial Loss(\mathbf{\mathcal{D}}^t, \Theta)}{\partial\Theta}$
			
			\CommentX{Compute Validation Performance}
			
			$\mathbf{v}^l \leftarrow [\mathbf{v}^l, Loss(\mathbf{\mathcal{D}}^v, \Theta)]$ 
		}
		
		\caption{\small {\bf I}nteractively- and {\bf I}ntegratively-connected {\bf D}eep {\bf R}ecurrent {\bf N}eural {\bf N}etwork (I$^2$DRNN) Learning}
		\label{ALG:ITERATIVEINF}
		
	\end{algorithm}

	\section{Information-Theoretic Analysis}
	\label{SEC:ITRNN}
	In this section, we develop an information-theoretic framework to examine the proposed model's learning behavior. 
	First, we explore the problem of spatio-temporal dependency learning from the perspective of information theory, and we then deduce the learning capacity of RNN. 
	On that basis, we analyze the i-CAP of the proposed model.
	Finally, we determine the necessary and sufficient configurations for the proposed model with respect to the given datasets.

	\subsection{Information-Theoretic Perspective on Spatio-Temporal Dependency Learning}
	We characterize the multi-scale dependency of spatio-temporal data using the concept of mutual information (MI)~\cite{cover2012elements}, which measures the degree of correlation between two random variables. 
	The larger the MI, the more uncertainty in one variable that can be eliminated when given the information for the other variable.

		To fully capture the spatio-temporal dependency, learning models should extract sufficient information from $\mathbf{X}_{[1:t]}$ into the representation $\mathbf{h}_t$ and then use $\mathbf{h}_t$ to infer $\mathbf{o}_t$. Therefore, we expect that the learned representation is informative with respect to predicting $\mathbf{y}_t$, i.e., maximizing $I(\mathbf{y}_t;\mathbf{o}_t)$. 
		In fact, the information that one can obtain about $\mathbf{y}_t$ through encoding $\mathbf{X}_{[1:t]}$ is upper bounded by $I(\mathbf{y}_t; \mathbf{X}_{[1:t]})$.
		According to the data processing inequality \cite{Beaudry:2012:IPD}, if three random variables form the Markov chain $A\rightarrow B \rightarrow C$, i.e., the conditional distribution of random variable $C$ depends only on $B$ and is conditionally independent of $A$ given $B$, then we have $I(A;B) \geq I(A;C)$. 
		In the case of RNN modeling, since $\mathbf{o}_t$ is generated from $\mathbf{X}_{[1:t]}$: $\mathbf{o}_t=RNN(\mathbf{X}_{[1:t]})$, $\mathbf{y}_t$ is independent of $\mathbf{o}_t$ given $\mathbf{X}_{[1:t]}$ and there is a Markov chain $\mathbf{y}_t \rightarrow \mathbf{X}_{[1:t]} \rightarrow \mathbf{o}_t$~\cite{tishby2015deep}. Therefore, we have 
		\begin{equation}\label{DataProcessingInequality}\small
		I(\mathbf{y}_t; \mathbf{X}_{[1:t]}) \geq I(\mathbf{y}_t;\mathbf{o}_t),
		\end{equation}
		i.e., $\exists \alpha \in [0,1]$ such that $I(\mathbf{y}_t;\mathbf{o}_t) = \alpha I(\mathbf{y}_t; \mathbf{X}_{[1:t]})$.
Accordingly, we can define the model's learning capacity $I(\mathbf{y}_t; \mathbf{o}_t)$: 
	
		\begin{definition}
			The {\bf \textit{information-based learning capacity (i-CAP)}}, $I(\mathbf{y}_t; \mathbf{o}_t)$, of a model $\mathcal{M}$ for a specific dataset can be defined as the proportion $\alpha \; ( 0 \leq \alpha \leq 1) $ of the information extracted by model from $X_{[1:t]}$ about $\mathbf{y}_t$:
			\begin{equation}
			\small
			\begin{split}
			& I(\mathbf{y}_t; \mathbf{o}_t) = \alpha I(\mathbf{y}_t;\mathbf{X}_{[1:t]}) \\
			\end{split}
			\label{EQU:DataComplexity}
			\end{equation}
		\end{definition}

%	To fully capture the spatio-temporal dependency, spatio-temporal models should extract sufficient information from $\mathbf{X}_{[1:t]}$ to predict $\mathbf{y}_t$.}
%	The maximum information of $\mathbf{X}_{[1:t]}$ to $\mathbf{y}_t$ is quantified as $I(\mathbf{y}_t;\mathbf{X}_{[1:t]})$; however, due to the data processing inequality \cite{Beaudry:2012:IPD}, the information would be lost to some extent. Let $D^{R}(\tau)$ be the information rate in recurrent with time tag $\tau$, and let $D^{X}$ be the information rate in input feed forward. 
%	Intuitively, $D^{R}(\tau)$ is the capacity of transferring information from $t- \tau$ to $t$, and $D^{X}$ is the capacity of transferring information from $x_t$ to $h_t$.
%	Accordingly, we can define the model's learning capacity $I(\mathbf{y}_t; \mathbf{o}_t)$: 
%
%	\begin{definition}
%		The {\bf \textit{learning capacity}}, $I(\mathbf{y}_t; \mathbf{o}_t)$, of a model $\mathcal{M}$ for a specific dataset can be defined as the proportion $\alpha \; ( 0 \leq \alpha \leq 1) $ of the maximal information extracted from $X_{[1:t]}$ to predict $\mathbf{y}_t$:
%		\begin{equation}
%		\small
%		\begin{split}
%		& I(\mathbf{y}_t; \mathbf{o}_t) = \alpha I(\mathbf{y}_t;\mathbf{X}_{[1:t]}) \\
%		\end{split}
%		\label{EQU:DataComplexity}
%		\end{equation}
%		where $\mathbf{o}_t = \mathcal{M}(X_{[1:t]}, \lambda,\eta )$ is the model prediction at time $t$.
%	\end{definition}

	\subsection{Learning Capacity of RNN}
	To deduce RNN's learning capacity, we first qualitatively demonstrate that the memory of previous inputs would leak when a new input is encoded in the hidden layers, and we then analyze the relationship between the recurrent information rate and the input information rate in conventional RNN.
	
	\vspace{0.2cm}
	\subsubsection{New Information vs. Historical Information}
	
	~
	
	Assume that $\mathbf{h}_{t-1}$ contains partial information about $\mathbf{x}_t$: $H(\mathbf{x}_{t}) = I(\mathbf{h}_{t-1}; \mathbf{x}_t) + H(\mathbf{x}_{t}|\mathbf{h}_{t-1})$, where $H(\mathbf{x}_{t}|\mathbf{h}_{t-1}) >0$, and thus there exists some function $g(\cdot)$ such that
	$\mathbf{x}_t = g(\mathbf{h}_{t-1})+\mathbf{\epsilon}_t,$
	where $E( \mathbf{h}_{t-1}\mathbf{\epsilon}^T_t)=0$ and $Var[\mathbf{\epsilon}_t| \mathbf{h}_{t-1}]>0$~\cite{belletti2018factorized}. 
	We can model the $\mathbf{\epsilon}_t$ as a Gaussian random white noise.
	In so doing, the input $\mathbf{x}_t$ is a Gaussian random variable.
	The MI between the hidden states in consecutive time steps is given in the following theorem:
	
	\begin{theorem}\label{Theorem1}
		Consider $\mathbf{x}_t \sim N(\mathbf{0}, diag(\sigma^2,\sigma^2,\cdots,\sigma^2))$, $\mathbf{W}$ is a full-rank matrix with the dimension of $dim(h)$, then we have:
		\begin{equation}
		\small
		\begin{split}
		I(\mathbf{h}_t,\tanh(\mathbf{U}\mathbf{x}_t+ \mathbf{W} \mathbf{h}_{t} +b)) =  \frac{dim(h)}{2}\log(1 +\frac{\lambda}{\sigma^2 \eta}),
		\end{split}	
		\label{EQU:HistI}
		\end{equation}
		where $\lambda$ and $\eta$ are the largest eigenvalues of $\mathbf{W^TW}$ and $\mathbf{U^TU}$, respectively.
	\end{theorem}
	
	\begin{proof}
		Consider two random variables $\mathbf{c}, \mathbf{d}\in\Re^{dim(H)}$ that are linearly correlated with each other by $\mathbf{c} = \mathbf{W}\mathbf{d} + \mathbf{\epsilon}$, where $\mathbf{\epsilon} \sim N(\mathbf{0},\mathbf{\Sigma})$, then according to~\cite{chechik2005information}, we have
		\begin{equation}
		\small
		I(\mathbf{c},\mathbf{d}) = \frac{1}{2}\log(|\mathbf{I} + \mathbf{\Sigma}^{-1/2}\mathbf{W}^T\mathbf{W}\mathbf{\Sigma}^{-1/2}|).
		\label{EQU:GauMI}
		\end{equation} 
		Becuase MI is invariant under reparameterization by homeomorphisms~\cite{kraskov2004estimating},
		we have the following:
		\begin{equation}
		\small
		I(\mathbf{h}_t,\tanh(\mathbf{U}\mathbf{x}_t+ \mathbf{W} \mathbf{h}_{t} +b))  = I(\mathbf{h}_t,\mathbf{U}\mathbf{x}_t+ \mathbf{W} \mathbf{h}_{t} ).
		\label{EQU:MI_invariant}
		\end{equation} 
		Following Proposition 1.4 in \cite{belletti2018factorized}, we combine Eqs.~(\ref{EQU:GauMI}) and (\ref{EQU:MI_invariant}) and thus obtain Eq.~(\ref{EQU:HistI}). 
		This completes the proof of Theorem~\ref{Theorem1}.
	\end{proof}
	
	Based on Theorem 1, we can evaluate how the new information affects the memory of RNN on previous inputs.
	
	\begin{corollary}
		Assume $H_x =  H(\mathbf{x}_t|\mathbf{X}_{[1:t-1]})/N^i$, where $N^i$ is the dimension of $\mathbf{x}_t$, then 
		\begin{equation}
		\small
		I(\mathbf{h}_t,\tanh(\mathbf{U}\mathbf{x}_t+ \mathbf{W} \mathbf{h}_t + \mathbf{b})) = \frac{dim(h)}{2}\log(1 + \frac{{2\pi e}\lambda}{ \eta e^{2H_x}}) .
		\label{EQU:IRecurrentStru}
		\end{equation}
	\end{corollary}
	
	\begin{proof}
		The RNN would extract information of $\mathbf{x}_t$ into the hidden layer via $\mathbf{U}\mathbf{x}_t$, and the information presenting at time step $t$ is $H(\mathbf{x}_t|\mathbf{X}_{[1:t-1]})$. 	
		For the Gaussian random variable $\mathbf{x}_t$, the relation between the variance $\sigma^2$ and entropy $H_x$ is given as follows:
		\begin{equation}
		\small
		\sigma^2 = \frac {e^{2H_x}}{{2\pi e}}.
		\label{EQU:Variance}
		\end{equation}
		The proof of Eq.~(\ref{EQU:Variance}) can be found in Page 244 of~\cite{cover2012elements}.
		By substituting it into Eq.~(\ref{EQU:HistI}), we complete the proof.
	\end{proof}
	
	Therefore, the information about the previous inputs would decrease when increasing $\eta H(\mathbf{x}_t|\mathbf{X}_{[1:t-1]})$. 
	In other words, more information about previous inputs would be lost if more new information is stored. 
	To store more information, we can increase the size of the hidden layers or increase $\lambda$, which can help preserve the long-range dependency~\cite{le2015simple,vorontsov2017orthogonality}.

	\vspace{0.2cm}
	\subsubsection{Recurrent Information Rate $D^{R}$ and Input Information Rate $D^{X}$}
	
	\
	
	Given the limited capacity of one hidden layer with a fixed number of hidden units, some information about previous inputs would be lost when new information is stored. 
	Based on Theorem 1 and Corollary 1, we have the following corollary.
	
	\begin{corollary}
		Assume that $\mathbf{h}_{t-1} \sim N(\mathbf{0}, \mathbf{I})$, then the input rate $D^{X}$ and the recurrent rate $D^{R}$ can be defined as follows:
		\begin{equation}
		\small
		\begin{split}
		D^{X} &\coloneqq  I(\mathbf{h}_t; \mathbf{x}_t) = \frac{dim(h)}{2}\log(1 + \frac{\eta }{\lambda}),\\
		D^{R}(1) & \coloneqq   I(\mathbf{h}_t; \mathbf{x}_{t-1}) = \frac{dim(h)}{2}\log(1 + \frac{{2\pi e}\lambda}{\eta e^{ 2H_x}}),\\
		D^{R}(\tau) & \coloneqq   I(\mathbf{h}_t; \mathbf{x}_{t-\tau}) = \frac{dim(h)}{2}\log(1 + \frac{{2\pi e}(1-\lambda)\lambda^{\tau}}{(1 - \lambda^{\tau})\eta e^{ 2H_x} }).\\
		\end{split}
		\label{EQU:Relation}
		\end{equation}
	\end{corollary}
	As shown in this corollary, if we increase the strength of $\eta$, the input rate $D^{X}$ will increase accordingly while the recurrent rate $D^{R}$ will decrease. Therefore, $D^{X}$ and $D^{R}$ are inversely correlated. In other words, in the RNN with a fixed number of hidden units, the input rate $D^{X}$ and the recurrent rate $D^{R}$ cannot both be high. 
	%Because the upper bound of $D^{R}$ can be determined by $D^{X}$, for simplicity, we denote $\hat{D}_\eta^{R}$ as the upper bound of $D^{R}$ with corresponding $D^{X}$.

	In this paper, we assume that $0 \leq \lambda < 1$, because the spectral radius of $\mathbf{W}$ tends to be smaller than 1 for compressing the long-range information to match the real-world information pattern shown in Fig.~\ref{Figure:IllustrationDataComplexiy}(b) and to avoid the gradient exploding~\cite{pascanu2013difficulty}. 
	Under this mild assumption, we show that the memory regarding previous inputs would decay exponentially.
\begin{proposition}
		$I(h_t; h_{t-\tau})$ decays exponentially in the long-range dependency if $0 \leq \lambda < 1$.
	\end{proposition}
	
	\begin{proof}
		From Eq.~(\ref{EQU:Relation}), we know that when $\tau$ becomes large, $(\lambda)^\tau$ becomes very small. Thus using Taylor expansion at the point $(\lambda)^\tau = 0$, we have:
		\begin{equation}
		\small
		D^{R}(\tau) \approx \frac{dim(h)}{2}\cdot\frac{{2\pi e}(1-\lambda) \lambda ^\tau}{\eta e^{ 2H_x}  } = O( \lambda ^\tau).
		\end{equation} 
		So the information decays exponentially with time lag $\tau$. 
	\end{proof}

	\begin{figure}[!t]
		\centering
		\setlength\tabcolsep{1.0pt}
		\begin{tabular}[t]{cc}
			\includegraphics[width=0.63\linewidth]{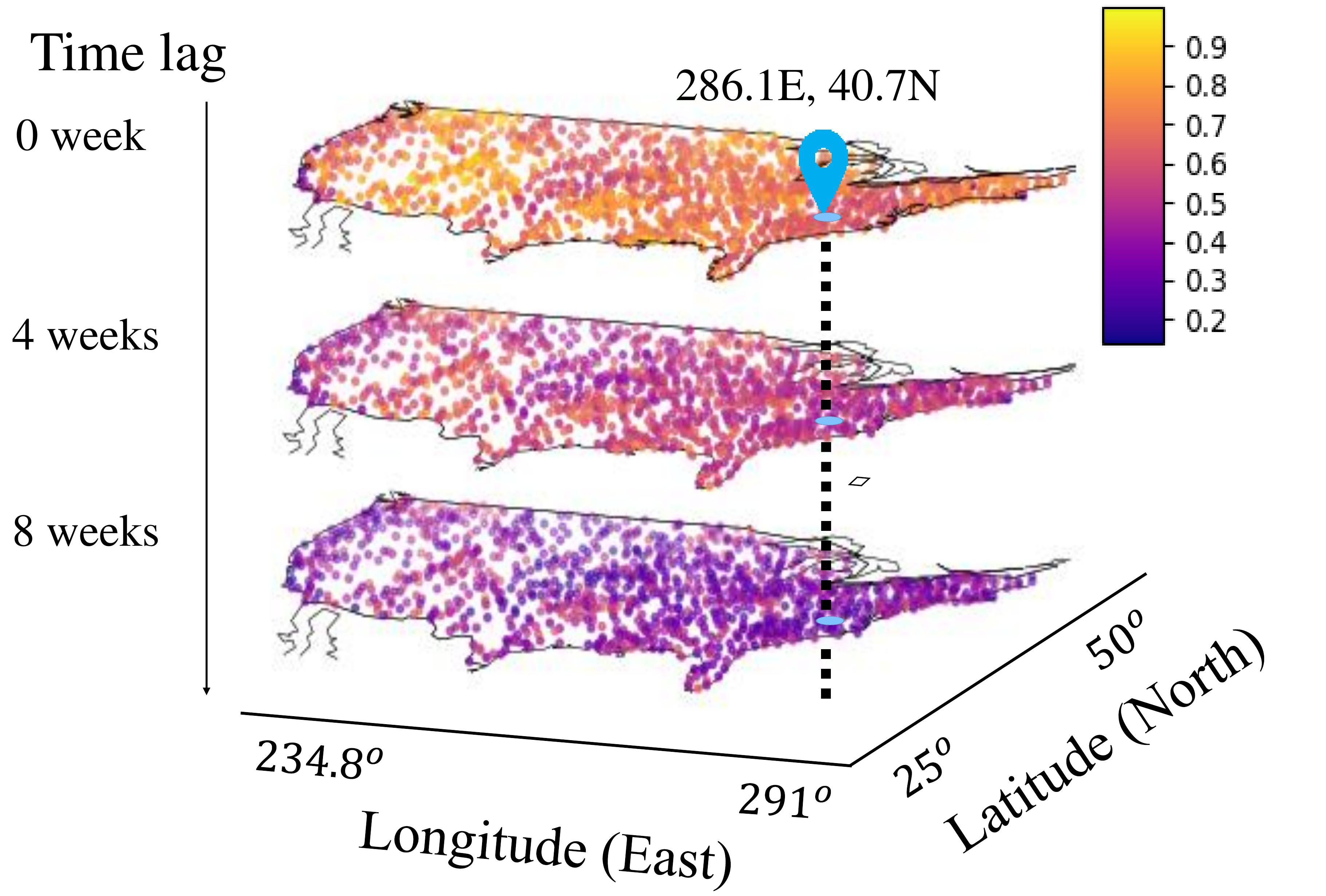}\label{Figure:IllustrationSpatialDataComplexiy} &
			\includegraphics[width=0.28\linewidth]{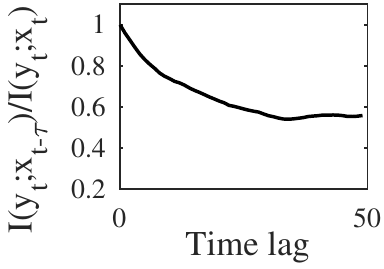}\label{Figure:IllustrationTemporalDataComplexiy} \\
			\footnotesize  (a) &\footnotesize \hspace{0.45cm} (b) \\
		\end{tabular}
		\caption{Information-theoretic perspective on multi-scale spatio-temporal dependency. (a) MI between the target variable $Y_{t,j}$ and input covariates $X_{t-lag,k}$ in a climate dataset, where $j$ represents the location of (286.1E, 40.7N) marked with a pointer and $lag$ indicates the time lag of $X_{t-lag,k}$ with respect to $Y_{t,j}$. The colour intensity in each location $k$ visualizes the value of $\sum_{t=lag}^{T}I(Y_{t,j};X_{t-lag,k})/T$, with $lag = 0$, $4$, and $8$ weeks in the top, medium, and bottom layers, respectively.
		(b) Normalized MI between the target variable $\mathbf{y}_t$ and the lagged input covariates $ \mathbf{x}_{t-\tau}$, i.e., $I(\mathbf{y}_t; \mathbf{x}_{t-\tau})/I(\mathbf{y}_t; \mathbf{x}_{t})$, with varying time lags $\tau$ in a traffic dataset.} 
		\label{Figure:IllustrationDataComplexiy}
		\vspace{-0.2cm}
	\end{figure}

	\subsection{Information-based Learning Capacity (i-CAP) of I$^2$DRNN}
	
	\
	
	In this subsection, we examine the capacity of I$^2$DRNN to learn multi-scale spatio-temporal dependency.
	We first intuitively explain the advantage of the structure of I$^2$DRNN over the stacked RNN in capturing multi-scale information. We then quantitatively demonstrate the superiority of I$^2$DRNN by comparing its capacity with that of the stacked RNN for explaining the potential performance gains.
	
	\vspace{0.2cm}
	\subsubsection{Structures: I$^2$DRNN vs. Stacked RNN}
		
	Multi-scale spatio-temporal dependency exists in many real-world datasets. Fig.~\ref{Figure:IllustrationDataComplexiy}(a) shows the MI between the target variable and input covariates with various locations and different time lags in a climate dataset\footnote{Please refer to Section~\ref{SEC:DD} for the details of this climate dataset.}. 
	%Here $j$ and $k$ are spatial indices, with $j$ is the index of the location of (286.1E, 40.7N) marked with a pointer, and $lag$ indicates the time lag of the input covariates with respect to the target variable. 
	The colour intensity in each point indicates the MI between the input covariates at this point and the target variable at the marked point (286.1E, 40.7N) in the top layer.
	%For each location $k$ and a specific $lag$, we calculate $\sum_{t=lag}^{T}I(Y_{t,j};X_{t-lag,k})/T$ and represent the value using the color intensity in the corresponding location in Fig.~\ref{Figure:IllustrationDataComplexiy}(a), with the top, medium, and bottom layers indicating $lag = 0$, $4$, and $8$ weeks, respectively.
	Fig.~\ref{Figure:IllustrationDataComplexiy}(b) shows the MI between the target variable and the lagged input with varying time lags $\tau$ in a traffic dataset\footnote{Please refer to Section~\ref{SEC:DD} for the details of this traffic dataset.}. 
	%For each point with lag $\tau$, we first calculate $I(\mathbf{y}_t; \mathbf{x}_{t-\tau})$ and we normalized it as $I(\mathbf{y}_t; \mathbf{x}_{t-\tau})/I(\mathbf{y}_t; \mathbf{x}_{t})$.
	As shown in these two figures, in addition to the information about the target variable among the input features in nearby locations and time steps, considerable amounts of information are in faraway locations and time steps, which should be considered in an aggregate manner.

	Stacked RNN performs well in many applications. %because it allows various layers to work at different time scales and thus becomes parameter efficient.
	However, in the following, we interpret the limitations of stacked RNN in balancing the input and memories, and demonstrate the corresponding improvement to overcome these limitations.

	\begin{figure}[!t]
		\centering
		\setlength\tabcolsep{1.0pt}
		\begin{tabular}[t]{cc}
			\hspace{0.9cm} \hspace{-0.2cm}\includegraphics[width=0.305\linewidth]{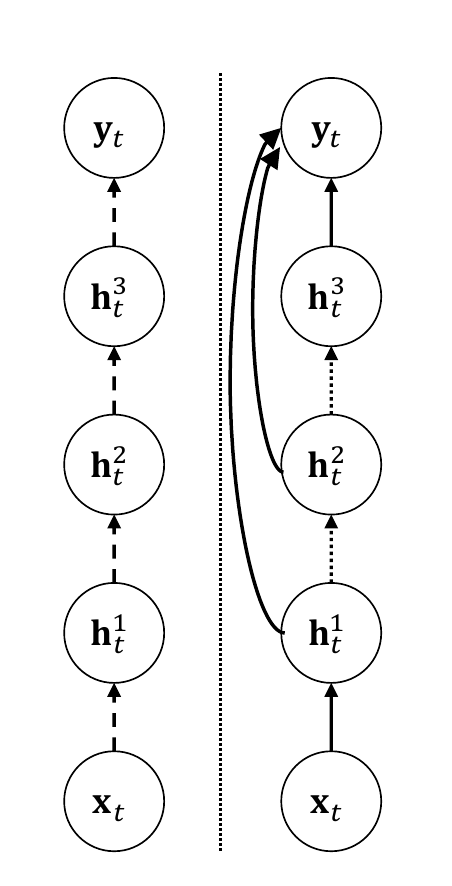}\label{Figure:Extension(a)}  &
			\hspace{0.0cm} \includegraphics[width=0.55\linewidth]{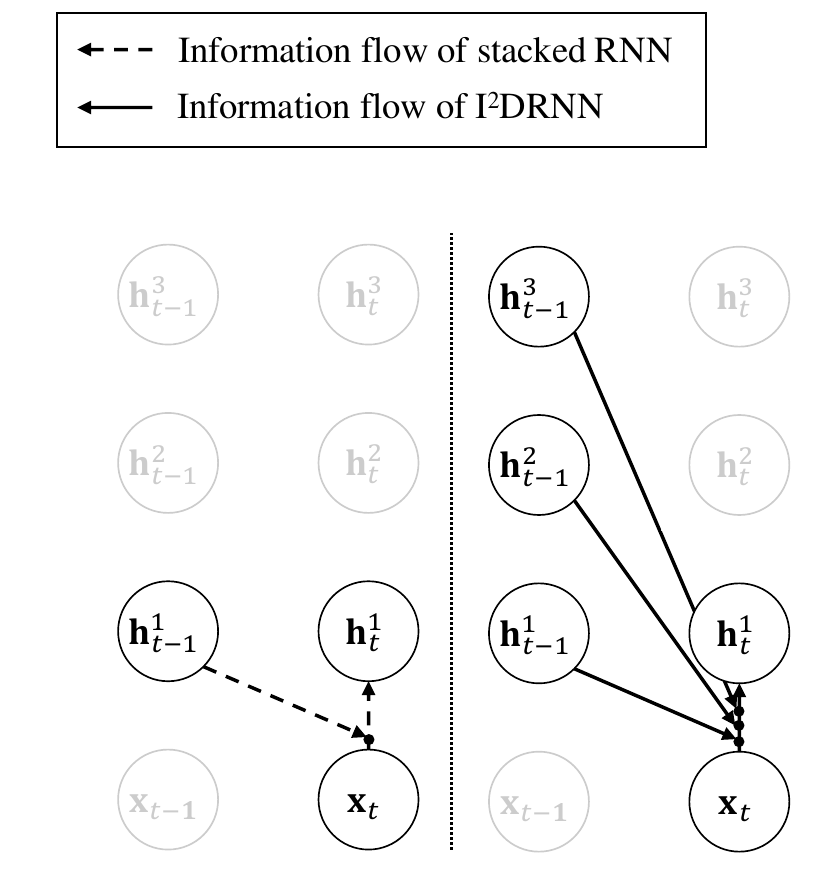}\label{Figure:Extension(b)}  \\
			\hspace{0.65cm} \footnotesize  (a) &\hspace{0.55cm} \footnotesize (b) \\
		\end{tabular}
		\caption{The advantage of I$^2$DRNN's structure over a stacked RNN's structure. (a) Compared to stacked RNN (left), the fully connected output module in I$^2$DRNN (right) provides shortcuts from $\mathbf{h}_{t}^1, \mathbf{h}_{t}^2, \mathbf{h}_{t}^3$ to $\mathbf{y}_t$ to preserve short-term and long-term memories in different layers. (b) Stacked RNN (left) uses only the information in $\mathbf{h}_{t-1}^1$ in the previous time step; I$^2$DRNN (right) uses all available information in $\mathbf{h}_{t-1}^1, \mathbf{h}_{t-1}^2, \mathbf{h}_{t-1}^3$ in the previous time step to control the input to prevent redundant information filling in the memory bank.} 
		\vspace{-0.2cm}
		\label{Figure:Extension}
	\end{figure}

	\begin{itemize}
		\item First, in the procedure of predicting $\mathbf{y}_t$, the information $I(\mathbf{y}_t;\mathbf{x}_{t})$ would go through layer $1$ to layer $L$ in the stacked RNN, as illustrated in the left of Fig.~\ref{Figure:Extension}(a). As shown in Eq.~(\ref{EQU:Relation}), more memory about past inputs would be lost when more new information is inputted. In other words, the long-term memory in the higher layers would be lost due to the flow of $I(\mathbf{y}_t;\mathbf{X}_{[1:t]})$. 
		%However, the higher layers are not needed to store the entire information about $\mathbf{x}_t$ in tracking the long-term memory to predict the target variables in the far future. 
		In contrast to the stacked RNN, I$^2$DRNN provides the shortcut from each hidden layer to the output layer, as illustrated in the right of Fig.~\ref{Figure:Extension}(a), which directly utilizes the information from the bottom layers without sacrificing the long-term information in the higher layers.
		
		\item Second, as shown in Eq.~(\ref{EQU:Variance}), the variance of the input noise is determined by the conditional entropy $H(\mathbf{x}_t|\mathbf{h}_t)$. In the stacked RNN, the input of each layer is controlled simply by the hidden units of the same layer, as illustrated in the left of Fig.~\ref{Figure:Extension}(b). Because $H(\mathbf{x}_t| \mathbf{h}^{1}_{t-1}) \geq H(\mathbf{x}_t| \mathbf{h}^{2}_{t-1}, \mathbf{h}^{1}_{t-1}) $, I$^2$DRNN adds the feedback connection to reduce the variance of $\epsilon_t$ to preserve the information about historical inputs, as illustrated in the right of Fig.~\ref{Figure:Extension}(b).
	\end{itemize}

		From the information-theoretic perspective, these two improvements in the hidden module and output module could increase the recurrent feed forward rate $D^{R}$ by properly redesigning the multi-layer structure while keeping the same input feed forward rate $D^{X}$.
		By doing so, I$^2$DRNN is able to characterize multi-scale information, so as to achieve more accurate prediction.
	Note that this information-theoretic analysis will not be affected by the operations in the input module, as the $\mathbf{x}_t$ used for analysis, as shown in Eq.~(\ref{EQU:xt}), is a general notation that represents the output from the input module.

	\vspace{0.2cm}
	\subsubsection{i-CAP: I$^2$DRNN vs. Stacked RNN}
	
	\
	
To demonstrate the advantage of I$^2$DRNN over a stacked RNN, without loss of generality, we specifically compare the capacity of a two-layer I$^2$DRNN and a two-layer stacked RNN. 
We use superscripts $^{*}$ and $'$ to denote the notations in I$^2$DRNN and those in stacked RNN, respectively. Let the input coefficients and recurrent coefficients of the $l$-th ($l = 1, 2$) layer in I$^2$DRNN be $\eta^{*}_l$ and $\lambda^{*}_l$, respectively, and let those in the stacked RNN be $\eta'_l$ and $\lambda'_l$, respectively. 
	In the following, we will show that with the same input information rate $\bar{D}^X$, I$^2$DRNN is able to achieve higher $D^R(\tau)$ than stacked RNN, i.e., $D^{R*}(\tau) > D^R{'}(\tau)$.
		
	In stacked RNN, $\mathbf{o}_t$ is conditionally independent of other variables given the top hidden layer $\mathbf{h}_{t}^{2}$. 
	Let $\eta'_2$ and $\lambda'_2$ be the input and recurrent coefficients of the second layer that satisfy $\frac{dim(h)}{2}\log(1 + \frac{\eta'_2 }{\lambda'_2}) = \bar{D}^X$.
	 Since the information of $\mathbf{x}_t$ flows to $\mathbf{h}_t^2$ through $\mathbf{h}_t^1$, we have $I(\mathbf{h}^1_t; \mathbf{x}_t) \geq I(\mathbf{h}^2_t; \mathbf{x}_t)$. As a result, the second layer has a smaller input information rate and a larger recurrent information rate than the first layer, and thus preserves more long-term memory. In this case, the recurrent information rate of stacked RNN is given as follows:	 
	 \begin{equation}\small
	 D^R{'}(\tau) = \frac{dim(h)}{2}\log(1 + \frac{{2\pi e}(1-\lambda'_2){\lambda'_2}^{\tau}}{(1 - {\lambda'_2}^{\tau})\eta'_2 e^{ 2H_x} }). 
	\end{equation}

	In I$^2$DRNN, a shortcut is provided to transfer the information of $\mathbf{x}_t$ to $\mathbf{o}_t$.
	Let $\eta^{*}_1 = \eta'_2$ and $\lambda^{*}_1 = \lambda'_2$, then we have $D^{X*} = \bar{D}^X$. Assume that $\mathbf{h}^1_j$ contains sufficient information of $\mathbf{x}_j$.
	Moreover, we have $0 \leq \lambda'_2, \lambda^{*}_2 < 1$.
	Let $ \eta^{*}_2 = \eta'_2$ and $\lambda^{*}_2 > \lambda'_2$. Since $\tau$ is a positive integer, when $\tau = 1$, we have: 
		\begin{equation}\label{tau-1}
		\small
	\begin{split}
	D^{R*}(1) & = \frac{dim(h)}{2}\log(1 + \frac{{2\pi e}(1-\lambda^{*}_2){\lambda^{*}_2}}{(1 - {\lambda^{*}_2})\eta^{*}_2 e^{ 2H_x} }) \\
	%& = \frac{dim(h)}{2}\log(1 + \frac{{2\pi e}{\lambda^{*}_2}}{\eta^{*}_2 e^{ 2H_x} }) \\
	&> \frac{dim(h)}{2}\log(1 + \frac{{2\pi e}{\lambda'_2}}{\eta'_2 e^{ 2H_x} }) = D^R{'}(1).\\
	\end{split}
	\end{equation} 
	When $\tau > 1$, we have:
	\begin{equation}\label{tau-2}
	\small
	\begin{split}
		D^{R*}(\tau) & = \frac{dim(h)}{2}\log(1 + \frac{{2\pi e}(1-\lambda^{*}_2){\lambda^{*}_2}^{\tau}}{(1 - {\lambda^{*}_2}^{\tau})\eta^{*}_2 e^{ 2H_x} }) \\
		%&= \frac{dim(h)}{2}\log(1 + \frac{{2\pi e}{\lambda^{*}_2}^{\tau}}{(1 + {\lambda^{*}_2} + \cdots + {\lambda^{*}_2}^{\tau-1})\eta^{*}_2 e^{ 2H_x} }) \\
		&= \frac{dim(h)}{2}\log(1 + \frac{{2\pi e}}{(\frac{1}{{\lambda^{*}_2}^{\tau}} + \frac{1}{{\lambda^{*}_2}^{\tau-1}} + \cdots + \frac{1}{{\lambda^{*}_2}})\eta^{*}_2 e^{ 2H_x} }) \\
		&> \frac{dim(h)}{2}\log(1 + \frac{{2\pi e}}{(\frac{1}{{\lambda'_2}^{\tau}} + \frac{1}{{\lambda'_2}^{\tau-1}} + \cdots + \frac{1}{{\lambda'_2}})\eta'_2 e^{ 2H_x} }) \\
		& = D^R{'}(\tau).\\
	\end{split}
	\end{equation} 
	Combining Eqs.~(\ref{tau-1})-(\ref{tau-2}), we know that given the same $\bar{D}^X$, I$^2$DRNN is able to achieve higher $D^R(\tau)$ than stacked RNN.
	As a result, I$^2$DRNN can provide more accurate estimation on target variable via capturing more information from input data, with fixed amount of hidden units.
	%In this way, the second layer of  I$^2$DRNN is able to capture more long-range memory, making recurrent information rate of I$^2$DRNN bigger than that of stacked RNN with the same input information rate.

	\subsection{Configuration Determination}
	\label{SEC:Config}
	In order to gain more insight into the i-CAP of our model on various datasets, we analyze the model capacity when learning the data with parametric exponential  information functions.
	More importantly, based on the capacity analysis, we further introduce the concepts of necessary and sufficient configurations for the designed model with respect to the given datasets, and the way to determine them, which may provide some hints during the real-world model deployment.
	
	We first explain how to estimate the capacity of a model with certain number of layers and units. Based on that, we can select the range of model configuration for an acceptable performance. 
	Without loss of generality, we conduct the analysis on a two-layer I$^2$DRNN with $h_1$ hidden units in the first layer and $h_2$ hidden units in the second layer.
	We analyze the capacity to capture MI from an optimization point of view: the first layer is optimized to capture as much useful information as possible, and the second layer is then optimized to capture the information that is not captured by the first layer.
	
	\vspace{0.2cm}
	\subsubsection{Capacity Estimation}
	
	\
	
	\underline{Capacity of the First Layer:}~ Consider the data information function following the parametric form $g(\tau) = a*k^\tau$, where $g(\tau)$ is the mutual information between $\mathbf{y}_t$ and the time-lagged input $\mathbf{x}_{t-\tau}$ with $a \in \Re^{+}$ and $\tau \in [0,1]$. Such parametric function well matches the information curves of typical real-world spatio-temporal datasets with exponential information decay. The capacity of the first layer to capture the previous inputs can be written as:
	\begin{equation}
	\small
	\begin{split}
	&I(h_t^1;X_{[1:t]})  = \sum_{\tau=0}^{t-1} I(h_t^1;X_{t-\tau})\\
%	& = \frac{dim(h)}{2} \sum_{\tau=0}^{t-1} \log(1+\frac{\eta \lambda_1^\tau}{\lambda_1^{\tau+1} + \sum_{d=0}^{\tau-1} \lambda_1^d \eta e^{H_x} /\sqrt{2\pi e} } )  \\
	& = \frac{dim(h)}{2} \sum_{\tau=0}^{t-1} \log(1+\frac{{2\pi e} (1-\lambda_1) \eta \lambda_1^\tau}{{2\pi e}\lambda_1^{\tau+1}(1-\lambda_1) + \eta  e^{2H_x}(1 - \lambda_1^\tau )} ) .
	\end{split}
	\end{equation}
	We focus on the long-range behaviors and approximate $\lambda_1$ after model training by maximizing the following $f(\lambda_1)$:
	\begin{equation}
	\small
	f(\lambda_1) = \sum_{\tau=1}^{\infty} \log(1+\frac{{2\pi e} (1-\lambda_1) \eta \lambda_1^\tau}{{2\pi e}\lambda_1^{\tau+1}(1-\lambda_1) + \eta  e^{2H_x}(1 - \lambda_1^\tau )} )  *a k^\tau,
	\label{EQU:Capacity1}
	\end{equation}
	which encourage the curve pattern of information function captured by model being similar to the true distribution.
	
	By doing so, we make the information distribution captured by the first layer approach the true distribution.
	Because $\lambda^\tau \rightarrow 0$ when $\tau \gg 1$, by carrying out the Taylor expansion on $f(\lambda_1)$, we have the following:
	\begin{equation}
	\small
	f(\lambda_1) \approx  \sum_{\tau=1} ^{\infty}\frac{ a {2\pi e} (1-\lambda_1)  \lambda_1^\tau k^\tau }{ e^{2H_x}}    
	=    \frac{a {2\pi e} (1-\lambda_1) \lambda_1  k }{(1-\lambda_1 k) e^{2H_x}}.
	\end{equation}
	
	By taking the derivative of $f(\lambda_1)$ and setting it to $0$, we can find the solution:
	\begin{equation}
	\small
	f'( \lambda_1) = \frac{k(k\lambda_1^2-2\lambda_1+1)}{(1-\lambda_1 k)^2} =0
	\Longrightarrow  \lambda^*_1 = \frac{2-\sqrt{4-4k}}{2k}.
	\end{equation}

	\underline{Capacity of the Second Layer:}~ Similarly, we approximate $\lambda_2$ by optimizing the following $ f(\lambda_2) $:	
	\begin{equation}
	\small
	\begin{split}
	f(\lambda_2) = \sum_{\tau=1} ^{\infty}\log(1+\frac{{2\pi e} (1-\lambda_2) \eta \lambda_2^\tau}{{2\pi e}\lambda_2^{\tau+1}(1-\lambda_2) + \eta e^{2H_x} (  1 - \lambda_2^\tau )} ) & \\ * (ak^\tau- h_1 q^\tau) ,&
	\end{split}
	\end{equation}
	where $q=\lambda^*_1$ and $h_1 q^\tau$ is the information captured by the first layer in lag $\tau$. 
	We make the information distribution captured by the second layer approach the remaining information that has not yet been captured by the first layer.
	Because $\lambda^\tau \rightarrow 0$ when $\tau \gg 1$, by carrying out the Taylor expansion on $f(\lambda_2)$, we have the following: 	
	    \begin{equation}
	\small
	\begin{split}
	f(\lambda_2) \approx  & \sum_{\tau=1}^{\infty} \frac{{2\pi e} (1-\lambda_2) \lambda_2^\tau }{ e^{2H_x}} (ak^\tau - h_1q^\tau) \\
	%= & \frac{a{2\pi e} (1-\lambda_2) \lambda_2  k }{(1-\lambda_2 k) e^{H_x}} -  \frac{h_1 {2\pi e} (1-\lambda_2) \lambda_2  q }{(1-\lambda_2 q) e^{H_x}}  \\
	= & \frac{{2\pi e} (1- {\lambda_2} ) {\lambda_2} ((h_1-a)\lambda_2kq +ak- h_1q) }{  e^{2H_x} (1-k{\lambda_2})(1-q{\lambda_2})} \\
	\approx & \frac{{2\pi e} (1- {\lambda_2} ) {\lambda_2} (ak- h_1q) }{  e^{2H_x} (1-k{\lambda_2})(1-q{\lambda_2})} .
	\end{split}
	\end{equation}

	Note that $a$ and $h_1$ are the information within data in lag $0$ and that captured by the first layer in lag $0$, respectively, and we can assume that $a \approx h_1$.	
	Taking the derivative of $f(\lambda_2)$, we have the following:
	    \begin{equation}
	\small
	f'(\lambda_2) = \frac{(ak-h_1q) }{ \eta_2 e^{2\bar{\eta} H_x}}*\frac{\left( 1- 2\lambda_2  + (k+q+kq)\lambda_2^2   \right)}{\left( (1-k{\lambda}_2)(1-q{\lambda}_2) \right)^2}.
	\end{equation}
		Because $ f'(0) >0 $, $f'(1) < 0$ and $f''(\lambda_2)<0$ for $\lambda_2 \in (0,1)$, we have a maximal in $\lambda_2 \in (0,1)$, and the corresponding $\lambda_2$ can be obtained by solving the equation: $1- 2\lambda_2  + (k+q+kq)\lambda_2^2 = 0$. 
	
	\underline{Overall Capacity:}~
	As $\max(I(h_t^1; x_{t-\tau}), I(h_t^2; x_{t-\tau})) \leq I([h_t^1, h_t^2]; x_{t-\tau})  $, we can obtain the lower bound of the model capacity:
	\begin{equation}
	\small
	C = \sum_{\tau} \max(I(h_t^1; x_{t-\tau}), I(h_t^2; x_{t-\tau})),
	\end{equation}
	%where $I(h_t^1; x_{t-\tau})$ and $I(h_t^2; x_{t-\tau})$ can be calculated using Eq.~(\ref{EQU:ytot}) with the substitution of $\lambda_1^*$ and $\lambda_2^*$ into Eq.~(\ref{EQU:Relation}).
	The information bottleneck (IB) principle has been widely used to study the insight of deep neural networks 
	(DNN)~\cite{tishby2015deep,goldfeld19a,saxe2018information} and guide the learning of 
	DNN~\cite{Alex2017Deep,belghazi2018mine,li2019specializing}. The IB principle is applicable to various structures, such as the recurrent architectures~\cite{li2019specializing} and the ReLU activation function (Section 5.3 of \cite{belghazi2018mine}).
	According to the IB principle, the neural network learns the compact representation that contains sufficient information about inputs regarding the target variable to enhance the generalization~\cite{Igl2019Generalization}. 
	We can then estimate the i-CAP of the model as:
	\begin{equation}
	\small
	I(\mathbf{y}_t; \mathbf{o}_t) = \sum_{\tau} \min(g(\tau), \max(I(h_t^1; x_{t-\tau}), I(h_t^2; x_{t-\tau}))).
	\end{equation}
%	The transfer capacity can be represented as:
%	\begin{equation}
%	\small
%	\alpha=\frac{\sum_{\tau} \min(g(\tau), \max(I(h_t^1; x_{t-\tau}), I(h_t^2; x_{t-\tau}))}{\sum_{\tau} g(\tau)}.
%	\end{equation}
	
	\underline{Remarks:}~ In some situations, $X_t$ ($t=1,\cdots,T$) are not completely independent of one another. 	
	In terms of practical calculation, we assume that the RNN is randomly initialized and thus captures only short-range information. While training with back propagation over time, the RNN would gradually learn to capture longer-range information to minimize the empirical error. Therefore, in characterizing the task-specific requirement for capturing short- or long-range information, we apply the chain rule for MI: 
	$ I(h_t;X_{[1:t]}])  = \sum_{\tau=1}^{t} I(h_t;X_{t-\tau}|X_{[t-\tau+1:t]})$.
	In this way, the RNN needs only to capture the conditional MI at lag $\tau$, i.e., $I(h_t;X_{t-\tau}|X_{[t-\tau:t]})$, regardless of whether $X_t$ are independent of one another.
	A similar procedure can be used to estimate the capacity of the higher layers. 
	For example, to estimate the capacity of the third layer,
	let $ f(\lambda_3) = \sum_{\tau=1}^{\infty}\log(1+\frac{{2\pi e} (1-\lambda_3) \eta \lambda_3^\tau}{{2\pi e}\lambda_3^{\tau+1}(1-\lambda_3)+ \eta  e^{2H_x} ( 1 - \lambda_3^\tau )} ) * (ak^\tau- C(h_1,q_1,h_2,q_2, \tau) )$, where $C(h_1,q_1,h_2,q_2, \tau)$ denotes the capacity of the first two layers at time lag $\tau$. 
	We notice that the third layer would contribute to capturing the long-range dependency left out of the second layer. Similar to the procedure for calculating the approximate recurrent rate in the second layer, let $C(h_1,q_1,h_2,q_2, \tau) = h_2 \lambda_2^\tau$, we can obtain the recurrent coefficient in the third layer, $\lambda_3$, by solving the equation: $1- 2\lambda_3 + (k+\lambda_2)\lambda_3^2 - k\lambda_2 \lambda_3^3 = 0$. 
	
	\vspace{0.2cm}
	\subsubsection{Necessary and Sufficient Configurations}
	
	\
	
	Based on the estimated capacity of a specific configuration, when given a dataset, we can determine the model's minimum requirement (referred to as the necessary configuration) and maximum requirement (referred to as the sufficient configuration) to achieve an acceptable learning performance.

	\begin{figure}[!t]
		\centering 
		\setlength\tabcolsep{1.0pt}
		\begin{tabular}[t]{ccc}
			\multicolumn{3}{c}{\includegraphics[width=0.8\linewidth]{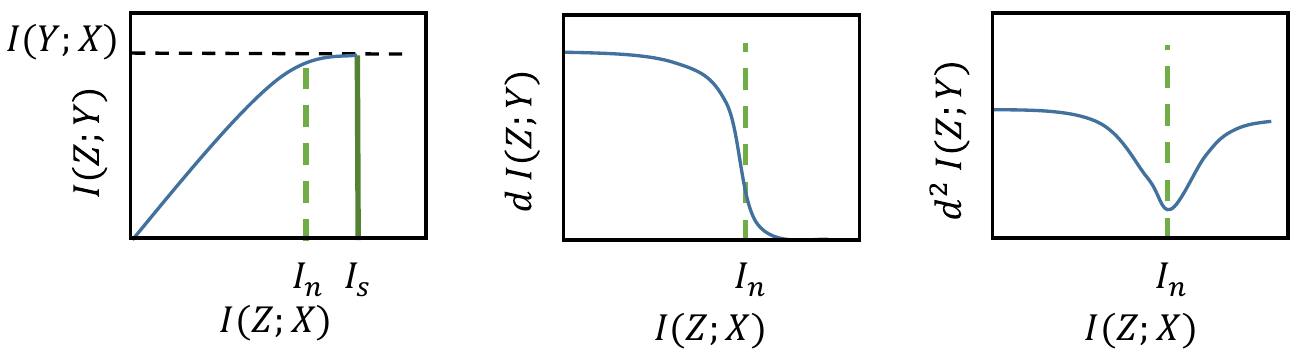}}\\
			\footnotesize  \hspace{1.60cm} (a) &\footnotesize \hspace{2.40cm} (b) &	\footnotesize \hspace{1.45cm} (c) \\
		\end{tabular}
		\caption{Necessary configuration ($I_n$) and sufficient configuration ($I_s$) of the designed model for a given dataset. (a) The $I(Z;Y)$ curve. (b) The first-order derivative of $I(Z;Y)$. (c) The second-order derivative of $I(Z;Y)$. } \label{Figure:configuration}
		\vspace{-0.2cm}
	\end{figure}

	According to the information bottleneck theory, the deep learning model is attempting to make an accurate prediction by learning the representation in as compact a manner as possible:
	$\arg_{p(Z|X)} \max I(Z;Y) - \beta I(X;Z)$,
	where $Z$ is the learned representation for prediction and $\beta$ is the Lagrange multiplier introduced to balance the complexity of the learned representation and the amount of the extracted useful information~\cite{tishby2015deep}.
	When the model is under-parametric, that is, when $I(Z;Y)$ is low, it cannot capture enough information for prediction, and the error rate is bounded by the conditional entropy, $H(Y|Z) = H(Y) - I(Z;Y)$.
	When the model is overparameterized, it may suffer from the over-fitting problem. 
	Therefore, it is important to achieve an appropriate balance between the representation complexity and the learning capacity.

	In our work, instead of tuning the trade-off parameter $\beta$ to achieve such a balance, we control the complexity of the learned representation while ensuring the learning capacity by configuring a proper number of hidden units for the model.
	Although it is difficult to determine in advance the exact best configuration for the model in terms of test data, we can, from the information-theoretic perspective, give a suitable range of model configurations. On the one hand, we can define the necessary configuration $I_n$ as follows:
	\begin{definition}
		The {\bf \textit{necessary configuration}}, $I_n$, of a model with respect to a given dataset is the configuration that minimizes the second-order derivative of $I(Z;Y)$.
	\end{definition}
	As clearly illustrated in the $I(Z;Y)$ curve of Fig.~\ref{Figure:configuration}(a), the information captured by the developed model increases steadily at a decreasing rate before reaching the turning point on the first-order derivative of the $I(Z;Y)$ demonstrated in Fig.~\ref{Figure:configuration}(b), which corresponds to the minimal point of the second-order derivative of $I(Z;Y)$ shown in Fig.~\ref{Figure:configuration}(c).
	Before such a turning point is reached, the rate of increase in the information captured is rather steady, and thus increasing the size of the configuration improves the model's learning capacity. Beyond the turning point marked by the vertical dashed line, the increase in information captured quickly diminishes to approach zero, which implies that it may not be as efficient to further increase the size of the model configuration.

	On the other hand, we can define the sufficient configuration $I_s$ as follows:
	\begin{definition}
		The {\bf \textit{sufficient configuration}}, $I_s$, of a model with respect to a given dataset is the minimal configuration that satisfies $I(Z;Y) =  I(Y;X)$.
	\end{definition}
	An illustrative example of $I_s$ is shown by the vertical solid line in Fig.~\ref{Figure:configuration}(a). 
	At this point, the model can capture all available information to predict $Y$, and increasing the size of the configuration no longer provides additional information.
	
	\section{Systematic Validation}
	\label{SEC:SYE}
	In this section, we conduct a series of experiments on both synthetic datasets and real-world PSTA tasks to validate the capacity of the proposed model to learn multi-scale spatio-temporal dependency and its learning behavior characterized by the information-theoretic framework.

	\begin{figure}[!t]
		\centering
		\setlength\tabcolsep{1.0pt}
		\begin{tabular}[t]{ccc}
			\multicolumn{3}{c}{\includegraphics[width=0.98\linewidth]{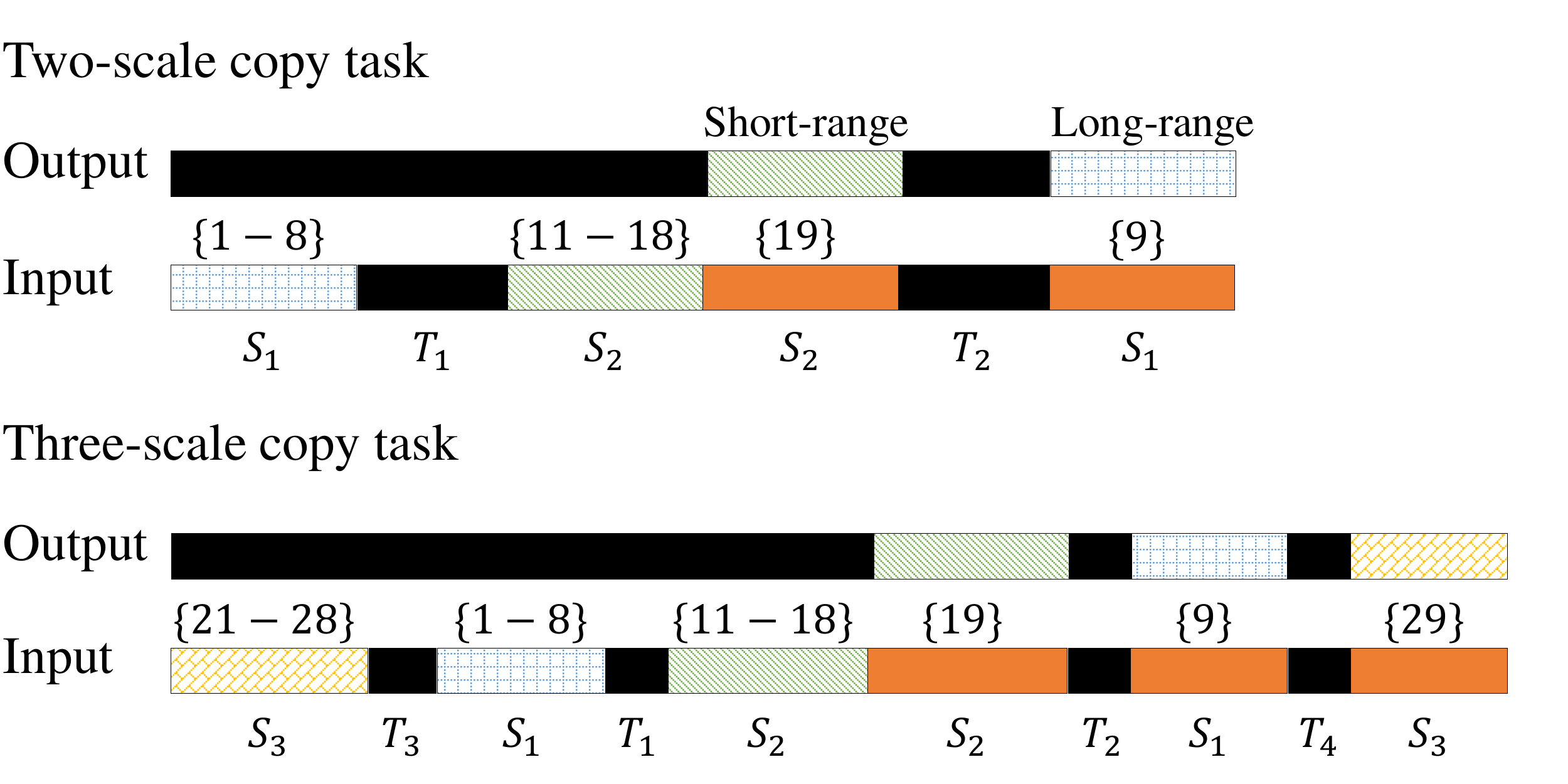} \label{FIG:msct}} \vspace{-0.1cm}\\
			\multicolumn{3}{c}{\footnotesize \hspace{0.25cm} (a)} \vspace{0.25cm}\\
			\includegraphics[width=0.315\linewidth]{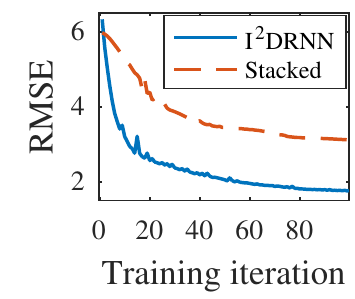} \label{FIG:SYNCurve} &
			\includegraphics[width=0.313\linewidth]{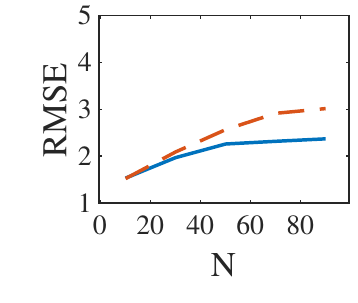} \label{FIG:Dim} &
			\includegraphics[width=0.313\linewidth]{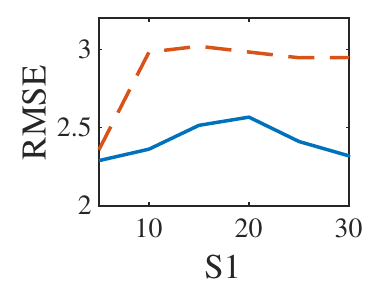}  \label{FIG:S1} \vspace{-0.1cm}\\
			\footnotesize  \hspace{0.5cm} (b) &\footnotesize \hspace{0.35cm} (c) &	\footnotesize \hspace{0.35cm} (d) \vspace{0.25cm}\\		
			\includegraphics[width=0.315\linewidth]{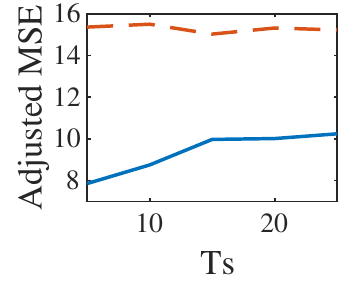}  \label{FIG:T} &
			\includegraphics[width=0.313\linewidth]{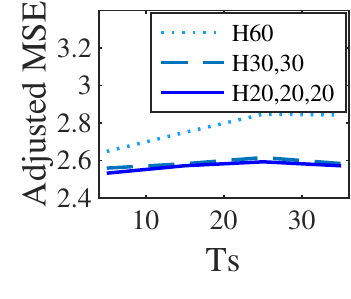} \label{FIG:HScale2}&
			\includegraphics[width=0.313\linewidth]{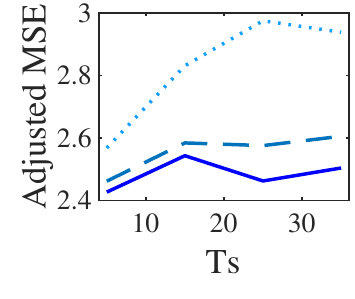} \label{FIG:HScale3} \vspace{-0.1cm} \\
			\footnotesize  \hspace{0.5cm} (e) &\footnotesize \hspace{0.35cm} (f) &	\footnotesize \hspace{0.35cm} (g) \\
		\end{tabular}
		\caption{Performance of I$^2$DRNN on synthetic datasets. 
			(a) Illustration of settings of the multi-scale copy task.
			(b) Performances of I$^2$DRNN and stacked RNN with fixed parameter setting on the two-scale copy task.
			(c)(d)(e) Performances of I$^2$DRNN and stacked RNN with varying $N$, $S1$ and $T_s$ on the two-scale copy task.
			(f) Performance of I$^2$DRNN with different configurations on the two-scale copy task.
			(g) Performance of I$^2$DRNN with different configurations on the three-scale copy task.} \label{FIG:Synthetic}
		\vspace{-0.0cm}
	\end{figure}

	\begin{table}[!t]
		\renewcommand{\arraystretch}{1.3} 
		\footnotesize  
		\caption{Overview of the datasets used in three real-world PSTA tasks: I) Disease Prediction, II) Climate Forecast, and III) Traffic Prediction. All tasks include data from heterogeneous sources with different spatial and temporal scales. The target variables are in bold face.}        	
		\label{TAB:dataset}	
		\centering	
		\vspace{-.3cm}
		
		\begin{tabular}[t]{|p{1.5cm}|p{2.0cm}|p{3.0cm}|p{2.0cm}|p{1.2cm}|p{1.0cm}|}
			
			\hline
			
			\hline
			
			\hline
			
			PSTA Tasks & Data Sources & Attributes & \specialcell{$\#$ of Spatial Points } & Time Interval &   \specialcell{Duration}\\
			
			\hline
			
			\hline
			
			\hline
			
			\multirow{3}{1.5cm}{I) {Disease\\ Prediction}} & \textbf{TYCHO} \textbf{Scarlet Fever} &  Number of Infection  & $59$ (states/territories) & $1$ week& \multirow{3}{1.2cm}{$24$ years} \\  
			
			\cline{2-5}
			
			& Climate & \specialcelll{Temperature, Precipitation} & $48$ (states) & $1$ month& \\
			
			\cline{2-5}
			
			& Covariate Risk Diseases & \specialcelll{ Influenza, Measles, Polio} & $59$ (states/territories) & $1$ week& \\
			\hline
			
			\hline
			
			\hline

			\multirow{4}{1.5cm}{II) {Climate \\Forecast}} & \textbf{NARR} & Air Temperature 2m & $100$ (grids)  & $1$ week& \multirow{4}{1.2cm}{$38$ years} \\
			
			\cline{2-5}
			
			& \multirow{2}{1.8cm}{NCEP}  & cpr., dlrf., dsrf., Precipitation rate, Temp max, Temp min, Pressure & \multirow{2}{1.4cm}{$390$ (grids)} &\multirow{2}{1.2cm}{$1$ month}& \\
			
			\cline{2-5}
			
			& USHCN & Min Temp, Max Temp, Precipitation &$100$ (stations)  & $1$ day& \\
			
			\hline
			
			\hline
			
			\hline
			
			\multirow{3}{1.5cm}{III) {Traffic \\Prediction}} & \textbf{Traffic} & Traffic Jam Index &  $68$ (sections) &$10$ mins& \multirow{3}{1.2cm}{$1$ month} \\
			
			\cline{2-5}
			
			& Weather & Rainfall &  \specialcelll{$11$ (districts)} & $3$ hours& \\
			
			\cline{2-5}
			
			& Air Quality & AQI  & $10$ (sites) & $1$ hour& \\
			\hline
			
			\hline  
			
			\hline
		\end{tabular}                   
		
	\end{table}

	\subsection{Learning Performance Evaluation}\label{SEC:LPE}
	First, we evaluate the performance of the proposed model by comparing it with other classical and state-of-the-art models, including the time series models, tensor-based learning models, deep neural network models, and RNN models, on both synthetic datasets with multi-scale dependency and real-world PSTA tasks with heterogeneous data sources.

	\vspace{0.2cm}
	\subsubsection{Synthetic Datasets}\label{SEC:SD}
	
	We design the \emph{Multi-scale Copy Task}, which includes data dependency at different scales, to validate the performance of I$^2$DRNN in capturing multi-scale dependency.
	The top of Fig.~\ref{FIG:Synthetic}(a) illustrates the setting of the two-scale copy task. 
	In this task, the input sequence consists of the following segments: 
	\begin{itemize}
		\item A segment of $S_1$ entries chosen randomly from the values of $1, \ldots, 8$; 
		\item A segment of $T_1$ entries with the value of $0$; 
		\item A segment of $S_2$ entries chosen randomly from the values of $11, \ldots, 18$; 
		\item A segment of $S_2$ entries with the value of $19$; 
		\item A segment of $T_2$ entries with the value of 0; and 
		\item A segment of $S_1$ entries with the value of $9$. 
	\end{itemize}
	Each output sequence is the same length as its corresponding input sequence and consists of the following segments:
	\begin{itemize}
		\item A segment of $S_1 + T_1 + S_2$ entries with the value of zero;
		\item A segment of $S_2$ entries the same as the third segment in the input sequence;
		\item A segment of $T_2$ entries with the value of zero; and 
		\item A segment of $S_1$ entries the same as the first segment in the input sequence. 
	\end{itemize}
	Here $S_1$ indicates the amount of information that is long-range correlated, and $T_1$ and $T_2$ determine the distance to transfer the long-range information.
	We compare the performance of I$^2$DRNN and the stacked RNN in predicting the output sequence, in terms of the root mean square error (RMSE): 
	$RMSE = \sqrt{\frac{\sum_{t=1}^{T^{test}}||\hat{\mathbf{y}}_t -\mathbf{y}_t||_2^2}{N*T^{test}}}$, 
	where $\hat{\mathbf{y}}_t$ and $\mathbf{y}_t$ are the predicted value and the ground truth at time step $t$ of the output sequence, respectively, $N$ is the number of output sequences in one sample, and $T^{test}$ is the length of the output sequences.
	The I$^2$DRNN and the stacked RNN both consist of two layers with 10 hidden units in each layer. 
	For each setting, we randomly generate 200 samples, among which $70\%$ are used for training and $30\%$ are used for testing. 
	Fig.~\ref{FIG:Synthetic}(b) shows the performance of I$^2$DRNN and the stacked RNN with settings of $N = 80, S_1= S_2 = 10, T_1 = T_2 = T_s = 15$. We can observe that I$^2$DRNN achieves a lower RMSE than the stacked RNN.
	To further evaluate the robustness of I$^2$DRNN, we vary the parameters $N, S_1, T_s$: $N = 10, 30, 50, 70, 90$, $S_1 = 5, 10, 15, 20, 25, 30$, and $T_s = 5, 10, 15, 20, 25$.
	When varying one parameter, we fix the others.
	Figs.~\ref{FIG:Synthetic}(c), \ref{FIG:Synthetic}(d) and \ref{FIG:Synthetic}(e) demonstrate the performance of I$^2$DRNN and the stacked RNN with various values of $N$, $S_1$, and $T_s$, respectively.
	The proposed model consistently outperforms the stacked RNN on various settings in terms of the prediction accuracy, thus validating the robustness of its superiority.
	Note that when $T_s$ increases, more $0$s are involved in the data sequences. To remove the impact in performance evaluation brought by these $0$s, we use the adjusted MSE to replace RMSE in the evaluation when varying the parameter $T_s$:  
	$Adjusted \, MSE = \frac{\sum_{t=1}^{T^{test}}||\hat{\mathbf{y}}_t -\mathbf{y}_t||_2^2} {N*T^{test} (S_1+S_2)/(S_1+S_2+Ts)}$.

	\vspace{0.1cm}
	Furthermore, we evaluate the effect of increasing the number of layers to capture the multi-scale dependency. 
	In the aforementioned two-scale copy memory task, there are two scales of dependency to be captured: a short-range one (between two $S_2$s) and a long-range one (between two $S_1$s).
	We make the task more challenging by increasing one more scale: the bottom of Fig.~\ref{FIG:Synthetic}(a) illustrates the setting of a three-scale copy memory task, presenting an additional scale of extremely long-range dependency (between two $S_3$s). 
	We evaluate the performance of I$^2$DRNN with three configurations: one layer with 60 hidden units (referred to as $H60$); two layers with 30 hidden units in each layer (referred to as $H30,30$); and three layers with 20 hidden units in each layer (referred to as $H20,20,20$).
%	\begin{itemize}
%		\item one layer with 60 hidden units (referred to as $H60$); 
%		\item two layers with 30 hidden units in each layer (referred to as $H30,30$); and 
%		\item three layers with 20 hidden units in each layer (referred to as $H20,20,20$). 
%	\end{itemize}
	
	Figs.~\ref{FIG:Synthetic}(f) and~\ref{FIG:Synthetic}(g) show the performance of I$^2$DRNN with various configurations on the two-scale and three-scale copy memory tasks, respectively. 
	Generally, more layers give better performance. 
	However, as shown in Fig.~\ref{FIG:Synthetic}(f), the performance of $H30,30$ is comparable with that of $H20,20,20$, possibly because the two-layer I$^2$DRNN is already sufficient to capture the short-range and long-range dependency in the two-scale copy memory task.
	For the three-scale copy memory task, because one more scale of extremely long-range dependency is included, the I$^2$DRNN with three-layers can learn the additional scale of dependency, which is not fully captured by the two-layer structure, and thus $H20,20,20$ further outperforms $H30,30$ in this task, as shown in Fig.~\ref{FIG:Synthetic}(g).

	\vspace{0.2cm}
	\subsubsection{Real-World PSTA Tasks}\label{SEC:PSTA}
	
	\
	
	We evaluate the performance of the proposed model on three representative real-world PSTA tasks: 
	I) disease prediction, II) climate forecast, and III) traffic prediction. 
	
	\vspace{0.1cm}
	\underline{Dataset Description:}~ 
	\label{SEC:DD}	
	Table~\ref{TAB:dataset} summarizes the statistics of datasets used for the three PSTA tasks.  
	The dataset of each task contains heterogeneous data sources with different spatial and temporal scales.%, which makes the spatio-temporal dependency between the target variable and covariates complex and difficult to learn.
	\begin{itemize}	
		\item Disease prediction: We use the state-wide scarlet fever data from the TYCHO dataset~\cite{tycho,matsubara2014funnel}, which includes weekly surveillance reports from the United States collected from 1928 to 1951. The scarlet fever dataset includes $59$ spatial regions and $1252$ time points. 
		Empirical studies have revealed the influence of climate conditions~\cite{brownstein2003climate} and co-evolving diseases on the outbreak of scarlet fever. 
		The monthly historical record of climate indices is collected by the National Centers for Environmental Information (\url{https://www.ncdc.noaa.gov/data-access/quick-links}) from 1895 to 2017. 
		
		\item Climate forecast: We use the climate dataset collected in the United States, which includes the air temperature data at the $2$-meter level from the North American Regional Reanalysis (NARR)  (\url{https://www.esrl.noaa.gov/psd/data/gridded/data.narr.monolevel.html}), the data about 7 climate variables from National Centers for Environmental Prediction (\url{https://www.esrl.noaa.gov/gmd/dv/data}), and the data about 3 climate variables from the United States Historical Climatology Network (USHCN) (\url{https://cdiac.ess-dive.lbl.gov/ftp/ushcn_daily}). We select the data during 1980-2017 ($38$ years in total) and in the spatial region of the United States. For the NARR temperature data, we average the original values into the weekly level and uniformly select 100 grids as the forecast target.
		
		\item Traffic prediction: We use an urban dataset collected in Shanghai that includes weather conditions, air quality indices and traffic jam indices for April 2015. 
		The weather data, air quality index data, and traffic jam index data were collected by the Shanghai Meteorological Bureau, the Shanghai Environmental Protection Bureau, and the Shanghai Urban and Rural Construction and Traffic Development Academy, respectively. 
		The dataset was released by the organizing committee of the Shanghai Open Data Apps (Season information technology Co. Ltd. Shanghai open data apps (2015). \url{http://soda.datashanghai.gov.cn/}). 
		We aim to predict the traffic index, which contains a total of $2160$ time points and $68$ spatial regions. 
	\end{itemize}

\begin{table*}[!t]
	\renewcommand{\arraystretch}{1.3}
	\footnotesize
	\caption{Comparison of the proposed model (I$^2$DRNN) and six representative models (GP~\protect\cite{rasmussen2006gaussian},  LSTM~\protect\cite{hochreiter1997long}, FS-RNN~\cite{mujika2017fast}, LRTL~\protect\cite{bahadori2014fast}, ST-ResNet~\protect\cite{zhang2018predicting}, and DA-RNN~\protect\cite{qin2017dual}) on three real-world PSTA tasks: I) Disease Prediction, II) Climate Forecast, and III) Traffic Prediction. The best performances among the seven models on different tasks are highlighted in bold face.}
	\label{TAB:Result}
\hspace{-2cm}
	\begin{tabular}{|m{1.8cm}|p{0.85cm}|p{0.9cm}|p{1.8cm}|p{1.6cm}|p{1.25cm}|p{1.9cm}|p{1.9cm}|l|}
		\hline
		
		\hline
		
		\hline
		\multirow{2}{*}{PSTA Tasks} &		\multirow{2}{*}{Criteria}	  & \multicolumn{7}{c|}{Methods} \\ 
		\cline{3-9}
		&      														& GP \cite{rasmussen2006gaussian}     & LSTM \cite{hochreiter1997long}            &   FS-RNN~\cite{mujika2017fast}  & LRTL \cite{bahadori2014fast}		& ST-ResNet \cite{zhang2018predicting}   & DA-RNN \cite{qin2017dual}   			&I$^2$DRNN                                                                               \\ 
		\hline
		
		\hline
		
		\hline
		\multirow{2}{2cm}{I) {Disease\\ Prediction}} & RMSE & $83.66$            						&  \specialcel{$38.47\pm	0.96$}			&		\specialcel{$51.54 \pm	0.85$}		& $52.89$           									& \specialcel{$50.20\pm0.92$     }   					& \specialcel{$42.66 \pm3.90$ 	} 								& \specialcel{ $\mathbf{37.13\pm 1.19}$   }  \\
		\cline{2-9}
		& MAE  & $51.17$            			 			& \specialcel{$25.54\pm1.04$}						&	\specialcel{$36.86 \pm	0.10$}	& $29.86$           									& \specialcel{$27.46 \pm0.54$                  }        & \specialcel{$32.27 \pm1.95$    			} 						&\specialcel{ $\mathbf{22.69\pm 1.04}$      }                               \\  
		\hline
		
		\hline
		
		\hline

		\multirow{2}{2cm}{II) {Climate \\Forecast}}          & RMSE &$ 2.81$       										& \specialcel{$8.83\pm 0.0007$ 	}	& \specialcel{$9.22 \pm	0.075$}	& $5.13$           												&  \specialcel{$4.61 \pm 0.13$			}				& \specialcel{$8.22 \pm 0.00008$ 	} 	& \specialcel{$\mathbf{2.71 \pm 0.07}$  }  \\
		\cline{2-9}
		& MAE  & $2.06$           										& \specialcel{$6.85\pm 0.0001$ 		}	& \specialcel{$6.96 \pm	0.055$}	& $4.19$         											& \specialcel{$3.691 \pm 0.11$ 	}              & \specialcel{$6.66 \pm 0.00005$  } 		&  \specialcel{ $\mathbf{2.00 \pm 0.06}$}  \\
		\hline
		
		\hline
		
		\hline
		\multirow{2}{2cm}{III) {Traffic \\Prediction}}            & RMSE & $6.38$            				   					&\specialcel{$5.30\pm0.12$   		}						&	\specialcel{$5.96 \pm	0.068$}	& $5.88$            											& \specialcel{$5.76 \pm 0.11$     }                       &\specialcel{ $5.15 \pm 0.08$ 	} 								& \specialcel{ $\mathbf{5.08\pm0.01}$   }\\
		\cline{2-9}
		& MAE  & $4.44$           				  						& \specialcel{$3.70\pm0.07$  					}				& 	\specialcel{$4.53\pm	0.086$}	 & $4.46$            											& \specialcel{$3.94 \pm0.13$       	}					 &  \specialcel{$3.69 \pm0.07$ 		} 							& \specialcel{ $\mathbf{3.65\pm 0.004}$  }     \\
		\hline
		
		\hline
		
		\hline
	\end{tabular}

\end{table*}

	%\vspace{0.15cm}	
	\underline{Experimental Settings:}~
	We compare the proposed I$^2$DRNN with six representative models: Gaussian Process (GP) model~\cite{rasmussen2006gaussian}, Long Short-Term Memory (LSTM) network~\cite{hochreiter1997long}, Fast-Slow Recurrent Neural Network (FS-RNN)~\cite{mujika2017fast}, Low-Rank Tensor Learning ({LRTL})~\cite{bahadori2014fast}, Spatio-Temporal Residual Network ({ST-ResNet}) \cite{zhang2018predicting}, and Dual-Stage Attention RNN ({DA-RNN})~\cite{qin2017dual}.
	We implement our model using the Pytorch~\cite{paszke2019pytorch}. We use vanilla RNN cell on the disease dataset and LSTM cell~\cite{hochreiter1997long} on the traffic and climate dataset, since the disease dataset has smaller sample size.
	We train our model using Adam optimizer~\cite{ADAM} with backpropagation through time (BPTT)~\cite{werbos1988generalization} with the step size of $0.001$.
	The size of hidden units in each layer is $90$ and the number of layers is $3$.
	The size of hidden units in LSTM, FS-RNN, ST-ResNet, and DA-RNN is the same as that in our model.
	For each dataset, we normalize the data to the range of $[0,1]$.
	We use the data in the first $64\%$ of time points for training, the following $16\%$ for validation, and the final $20\%$ for testing.
	We use two standard criteria for performance evaluation: the RMSE and the mean absolute error: $MAE = \frac{\sum_{t=1}^{T^{test}}|\hat{\mathbf{y}}_t -\mathbf{y}_t|}{N*T^{test}}$~\cite{chen2016learning}.
	The GP model is learned by a stable deterministic algorithm implemented in the Scikit package~\cite{scikit-learn}, and the results of the best kernel combination (RBF and dot-product kernels) are reported. 
	For disease prediction task, we use PCA to reduce the dimension of external features while preserving $99.9\%$ of the energy. 
	For LSTM, FS-RNN, ST-ResNet, DA-RNN, and I$^2$DRNN, we repeat the experiment 10 times with random initializations of the neural networks and report the average result.
	
	\underline{Results:}~ 
	Table~\ref{TAB:Result} shows the performances of the aforementioned seven methods on the three PSTA tasks. 
	As can be seen, I$^2$DRNN performs the best among all models.
	Specifically, I$^2$DRNN outperforms LRTL because LRTL uses the linear autoregressive model to capture the temporal dependency, while I$^2$DRNN uses the recurrent structure, which has a more powerful representation capacity.
	Compared with ST-ResNet, I$^2$DRNN integrates all available information rather than predefining a few attributes to be fed into the model, and it thus captures the multi-scale dependency of spatio-temporal data in a more comprehensive way.
	By incorporating the all-hidden-output connections and the feedback structure, I$^2$DRNN further improves the capacity of spatio-temporal dependency learning and thus performs better than LSTM, FS-RNN, and DA-RNN.

	\subsection{Analysis and Interpretation}
	In addition to the quantitative evaluation on the performance of I$^2$DRNN, we further analyze its learning behavior and interpret its learning results in the real-world context.
	
	\vspace{0.2cm}
	\subsubsection{Analysis of I$^2$DRNN's Learning Behavior}
	
	\
	
	The MI between hidden layers and the input layer on three PSTA tasks is shown in Fig.~\ref{FIG:Layer_INFO}. 
	The MI between the first layer and the input layer is high in small time lags and decays quickly as the time lag increases, indicating that the lower layer tends to capture short-range dependency. 
	The MI between the second/third layers of I$^2$DRNN and the input layer is relatively high when the time lag is large, indicating that the higher layer tends to capture long-range dependency.

	\begin{figure}[!t]
		\centering
		\setlength\tabcolsep{1.0pt}
		\begin{tabular}[t]{ccc}
			\hspace{-0.18cm} \includegraphics[width=0.31\linewidth]{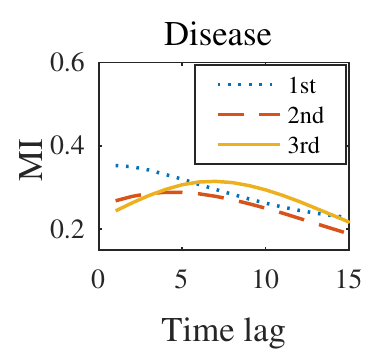} &
			\hspace{0.16cm} \includegraphics[width=0.31\linewidth]{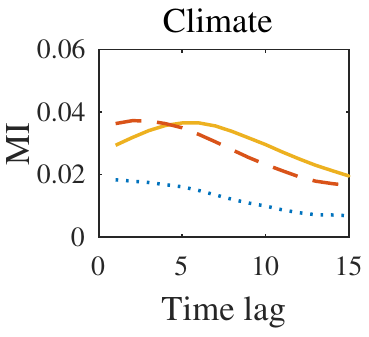} &
			\hspace{0.01cm} \includegraphics[width=0.31\linewidth]{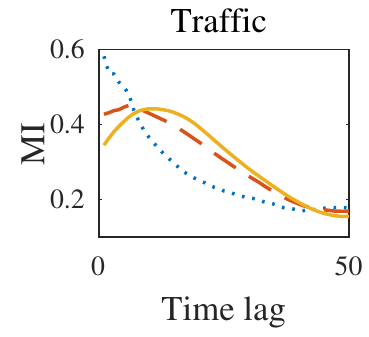} \\
			\footnotesize  \hspace{0.45cm} (a) &\footnotesize \hspace{0.70cm} (b) &	\footnotesize \hspace{0.55cm} (c) \\
		\end{tabular}
		\caption{MI between hidden layers and the input layer in various time lags, $I(h^l_t; X_{t-lag})$, on (a) disease prediction, (b) climate forecast, and (c) traffic prediction. The lower layers tend to capture the shorter dependency, while the upper layers tend to capture longer and coarser dependency.} \label{FIG:Layer_INFO}
		\vspace{-0.3cm}
	\end{figure}

	\begin{figure*}[!t]
		\centering
		\subfigure{\centerline{\includegraphics[width= 1.2\linewidth]{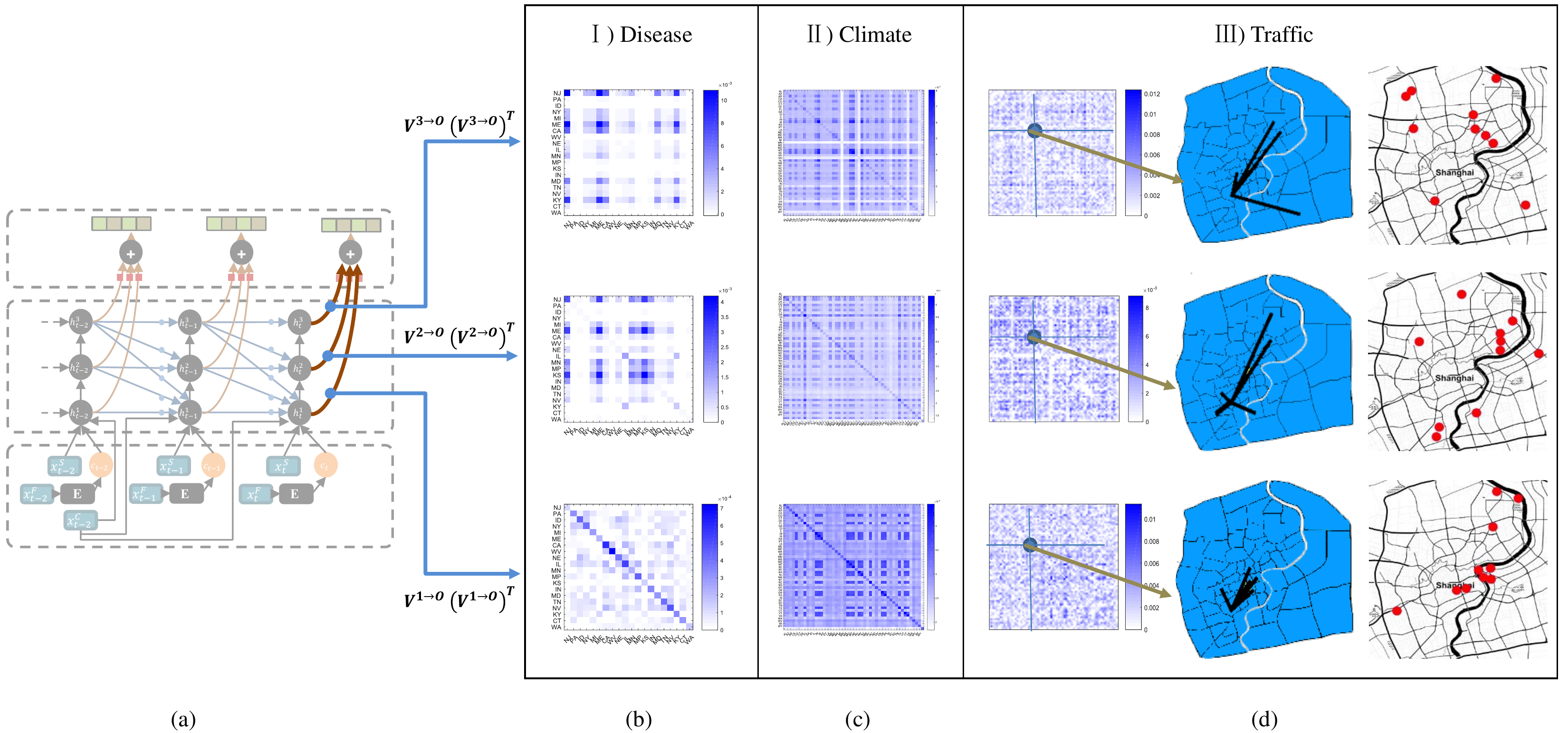}}}
		\vspace{-0.15cm}
		\caption{Multi-scale dependency among various regions on three real-world tasks. 
			(a) Correlations at various scales are calculated using the output weights of our model in different layers. 
			(b) Correlation matrices at various scales on the disease prediction task. 
			(c) Correlation matrices at various scales on the climate forecast task. 
			(d) Left column: correlation matrices at various scales on traffic prediction task; 
			middle column: top 5 correlated regions to Xujiahui district at various scales;
			right column: spatial distributions of dominant attributes to shape spatial correlations at various scales. Ten locations for the most dominant attribute with the largest POI counts in each scale are marked as red points.
			}
		\vspace{-0.20cm}
		\label{FIG:Covariance}
	\end{figure*}

	In the spatial context, we calculate the correlations among various locations at various scales. 
	As shown in Eq.~(\ref{EQU:OUTPUT}), the hidden units can be regarded as the hidden common factors of the output variables. 
	Fig.~\ref{FIG:Covariance}(a) recalls the architecture of the proposed I$^2$DRNN.
	Accordingly, we can calculate the correlations among various locations that arise from these common factors at the $l$-th scale ($l = 1, 2, 3$) using the output weights of I$^2$DRNN in the $l$-th layer: $\mathbf{Cov}^{l} =\mathbf{V}^{l\rightarrow O} {\mathbf{V}^{l\rightarrow O}}^{T}$. 
	Fig.~\ref{FIG:Covariance}(b) shows the correlations among the top 20 states in United States in terms of the number of scarlet fever cases, and Fig.~\ref{FIG:Covariance}(c) shows the correlations among the various states in terms of temperature. 
	The dynamic processes over these regions are self-correlated at the fine scale and that they tend to form some block-wise correlations at the coarser scales. 
	
	\vspace{0.2cm}
	\subsubsection{Interpretation of I$^2$DRNN's Learning Results}
	
	\
	
	For illustration purposes, we interpret the learning results of our model for the traffic prediction task in Fig.~\ref{FIG:Covariance}(d), as an example.
	First, similar to the results for the disease and climate tasks shown in Figs.~\ref{FIG:Covariance}(b) and~\ref{FIG:Covariance}(c), the traffic jam indices among various regions on the traffic prediction task also demonstrate the fine-scale correlations in the bottom layer and coarser-scale correlations in the higher layers, as shown in the left column of Fig.~\ref{FIG:Covariance}(d).
	Moreover, we show the top 5 correlated regions to Xujiahui district on the traffic prediction task in the middle column of Fig.~\ref{FIG:Covariance}(d).
	The correlated regions at the low level are near Xujiahui, while those at the higher levels are relatively distant. 
	This result indicates that our model can learn the spatio-temporal dependency at varying scales by using different layers, which helps to make predictions.
	
	To further mine the underlying mechanism that generates such traffic patterns, we collect the point-of-interest (POI) data in each region and count the number of POIs in 19 categories defined by the Baidu Map API SDK: food, hotel, shopping place, life service, beauty, tourist attraction, entertainment, sport, education, culture, medical service, car service, transportation, finance, real estate, corporation, government, doorway and natural features. 
	We then normalize the counts along each category and use Isometric Projection \cite{IsoPro} to identify the attributes that make the greatest contribution to shaping these traffic patterns.
	Isometric Projection learns a linear projection $W\in \Re^{19 \times ls}$ from the POI feature space to an $ls$-dimensional space where the distances between locations are similar to those in the left column of Fig.~\ref{FIG:Covariance}(d). 
	Finally, the summation of the weights for each category could be regarded as the importance of this category in shaping the spatial correlation in the corresponding scale. 
	The ten locations of the attributes that make the greatest contribution to the largest POI counts are shown as red points in the right column of Fig.~\ref{FIG:Covariance}(d).
	
	The discovered scales display a strong resemblance to the real-world context and conform to observations of the physical world. 
	We notice that the attribute that makes the greatest contribution to the first scale (the bottom hidden layer) is {\em tourist attraction}, that to the second scale (the middle hidden layer) is {\em real estate}, and that to the third scale (the top hidden layer) is {\em shopping place}. Tourist attractions are usually positioned near one another at central locations within a city, so they can be reached conveniently by many people and are much more densely populated than real estate districts, which can be seen across the city in various regions. 
	Yet, when compared to shopping places, which are widely distributed on every street corner in the city, real estate districts are more compact than the shopping places.
	
	\subsection{Necessary and Sufficient Configurations}
	We validate the necessary and sufficient configurations of our model as derived from the information-theoretic analysis on both synthetic datasets and real-world PSTA tasks.	
	
	\begin{table}[!t]
		\centering
		%\vspace{-0.2cm}
		\caption{Necessary, sufficient, and the best configurations of I$^2$DRNN with different $D$s in Fractional ARIMA datasets.}
		\renewcommand{\arraystretch}{1.3}
		\begin{tabular}{|c|c|c|c|}
			\hline
			
			\hline
			
			\hline
			D         & 20  & 40   & 60   \\ 
			\hline
			
			\hline
			
			\hline
			Necessary Configuration       & 160 & 320  & 500  \\ \hline
			Sufficient Configuration       & 500 & 1000 & 1400 \\ \hline
			The Best Configuration & 280 & 560  & 1000 \\ 
			\hline
			
			\hline
			
			\hline
		\end{tabular}
		\vspace{-0.2cm}
		\label{TAB:SynHHRS}
	\end{table}

	\vspace{0.1cm}
	\subsubsection{Synthetic Datasets}
	
	\
	
	We use Fractional ARIMA($p,d,q$) model to generate $D$ long-range dependent time series. 
	We set $p = 0.9, d = 0.1, q = 0$ and $D = 20, 40, 60$ in our experiment.
	Obviously, larger $D$ indicates greater complexity of the dependency.
	For each $D$, we measure the time-lagged MI and determine the range of hidden size accordingly. 
	Without loss of generality, we set the number of layers to 2.
	Following the work in~\cite{tishby2015deep}, we use the bin method to estimate the MI. 
	As shown in Table~\ref{TAB:SynHHRS}, the best configurations corresponding to the minimum RMSE always fall within the range between the necessary and sufficient configurations, thus validating the effectiveness of our information-theoretic analysis. \\
	
	\vspace{2cm}
	\subsubsection{Real-World PSTA Tasks}
	
	\
	
	We further evaluate the effectiveness of our information-theoretically derived model configurations on disease prediction, climate forecast, and traffic prediction tasks used in Section~\ref{SEC:PSTA}.
	The numbers of layers of I$^2$DRNN used for these three datasets are 1, 3, and 2, respectively. 
	
	Fig.~\ref{FIG:sizerange} shows the configuration-capacity on the three tasks (the top row), as well as the first- and second-order derivatives of the curve (the middle and bottom rows, respectively).
	The range between the necessary configuration and the sufficient configuration is shaded in green, and the best configuration with respect to the test performance is highlighted with a vertical red line. 
	As with synthetic datasets, the best configurations always fall within the recommended interval for all tasks.
	This fact is consistent with the results of our information-theoretic analysis, which provides a way to answer an open question in deep learning, i.e., how to determine the range of desirable configurations of a certain model for the given datasets.
	
	\begin{figure}[!t]
		\centering
		\setlength\tabcolsep{1.0pt}
		\includegraphics[width=0.8\linewidth]{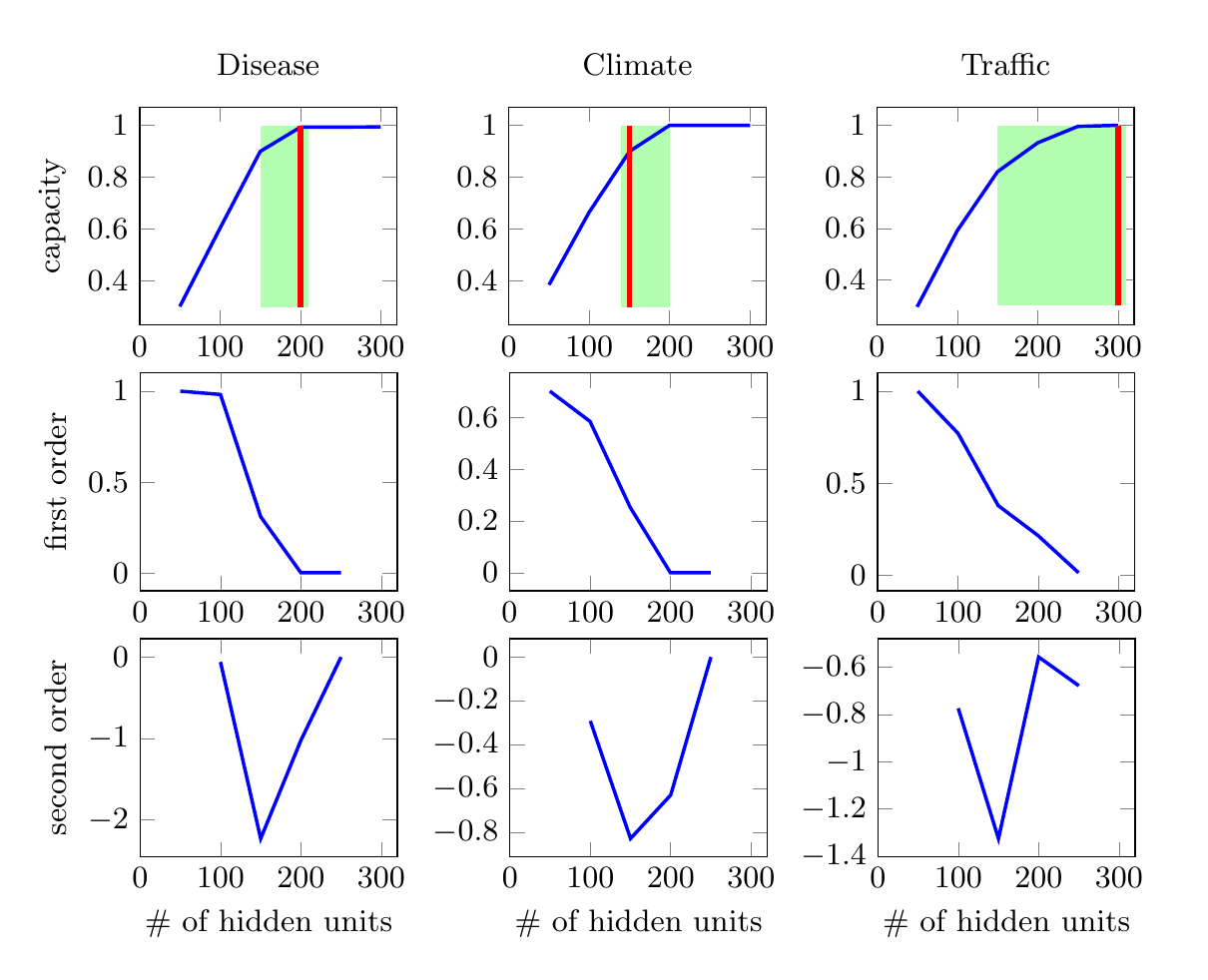} 
		\begin{tabular}[t]{ccc}
			%\hspace{-0.08cm} \includegraphics[width=0.31\linewidth]{figures/range/CAP1_feverfig_v.pdf} &
			%\hspace{0.06cm}
			%\includegraphics[width=0.31\linewidth]{figures/range/CAP1_climate2_v.pdf} &
			%\hspace{0.08cm} \includegraphics[width=0.31\linewidth]{figures/range/CAP1_traffic_v.pdf} \\
			%\hspace{-0.08cm} \includegraphics[width=0.31\linewidth]{figures/range_smooth/CAP2_fever.pdf} &
			%\hspace{0.08cm} \includegraphics[width=0.31\linewidth]{figures/range_smooth/CAP2_climate2.pdf}&
			%\hspace{0.08cm} \includegraphics[width=0.31\linewidth]{figures/range_smooth/CAP2_traffic.pdf}  \\
			
			%\hspace{-0.08cm} \includegraphics[width=0.31\linewidth]{figures/range_smooth/CAP3_fever.pdf} &
			%\hspace{0.08cm} \includegraphics[width=0.31\linewidth]{figures/range_smooth/CAP3_climate2.pdf}&
			%\hspace{0.08cm} \includegraphics[width=0.31\linewidth]{figures/range_smooth/CAP3_traffic.pdf} \\
			\footnotesize  \hspace{1.6cm} (a)\hspace{1.4cm} &\footnotesize \hspace{1.0cm} (b) \hspace{1cm}&	\footnotesize \hspace{1cm} \hspace{1.1cm} (c) \hspace{1cm} \\
		\end{tabular}
		\caption{Necessary, sufficient, and the best configurations of I$^2$DRNN on (a) disease prediction, (b) climate forecast, and (c) traffic prediction tasks. \emph{Top row}: The configuration-capacity curve. The range between the necessary configuration and the sufficient configuration is shaded in green, and the best configuration with respect to the test performance is highlighted as a vertical red line. \emph{Middle row}: First-order derivative of the configuration-capacity curve. \emph{Bottom row}: Second-order derivative of the configuration-capacity curve.} \label{FIG:sizerange}
		\vspace{-0.2cm}
	\end{figure}

	\section{Conclusions}
	\label{SEC:CON}
	In this study, we investigated an important and challenging problem in deep learning for PSTA: Given a learning dataset with multi-scale spatio-temporal dependency, how to theoretically guide the specific design, analytical understanding, and empirical validation of a deep learning model, so that its behaviors and performance can be guaranteed and explained.
	To address this problem, we first presented an I$^2$DRNN model that can incorporate data from heterogeneous sources and use the hierarchical recurrent structure to characterize the complex spatio-temporal dependency at varying scales to make predictions.
	We then introduced an information-theoretic framework to quantitatively characterize the i-CAP of the model and analytically derive the necessary and sufficient configurations of the model with respect to the given datasets.
	Finally, we conducted comprehensive experiments to validate the effectiveness of our model and examine the consistency with our information-theoretic analysis.
	%This study represents an important effort to demonstrate how to specifically design, analytically understand, and systematically validate a desirable deep model for a complex learning task. 
	Along this direction, in our future work, we plan to extend the current study to other representative deep architectures and learning tasks.
	
%	This study represents an important effort to demonstrate how to specifically design, analytically understand, and systematically validate a desirable deep model for a complex learning task.
%	Along this direction, in our future work, \textcolor{blue}{we plan to extend the current study to other representative RNN models such as Gated Feedback RNN (GF-RNN)~\cite{chung2015gated}, Gated Orthogonal Recurrent Units (GORU)~\cite{jing2019gated}, and Rectified Linear Units (ReLU)~\cite{le2015simple}, which have been shown to be efficient in training and thus are especially desirable in solving large-scale PSTA tasks.} 

	%Reference-------------------------------------------------------------------------
	\bibliographystyle{abbrv}
	\bibliography{arxiv}

\end{document}